\documentclass{article}

\usepackage{microtype}
\usepackage{graphicx}
\usepackage{subcaption}
\usepackage{booktabs} 
\usepackage{adjustbox}

\usepackage{hyperref}

 \usepackage[preprint]{icml2026}

\usepackage{amsmath}
\usepackage{amssymb}
\usepackage{mathtools}
\usepackage{amsthm}
\usepackage{multirow}

\usepackage[capitalize,noabbrev]{cleveref}

\theoremstyle{plain}
\newtheorem{theorem}{Theorem}[section]

\newtheorem{lemma}[theorem]{Lemma}

\theoremstyle{definition}

\theoremstyle{remark}

\usepackage[textsize=tiny]{todonotes}

\icmltitlerunning{Submission and Formatting Instructions for ICML 2026}

\begin{document}

\twocolumn[
  \icmltitle{Harnessing the Potential of Optimizing Data Mixtures \\
  	 via Bayesian Domain Reweighting}



  \icmlsetsymbol{equal}{*}

  \begin{icmlauthorlist}
    \icmlauthor{Xiang Yuan}{yyy}
    \icmlauthor{Kaiqing Lei}{comp}
    \icmlauthor{Zhenyu Jin}{yyy}
    \icmlauthor{Jun Shu}{yyy}
    \icmlauthor{Deyu Meng}{yyy}
    \icmlauthor{Zongben Xu}{yyy}
  \end{icmlauthorlist}
  \icmlaffiliation{yyy}{School of Mathematics and Statistics, Xi'an Jiaotong University}
  \icmlaffiliation{comp}{Northwest University, Shaanxi, China}
  
  \icmlcorrespondingauthor{Jun Shu}{xjtushujun@gmail.com}

  \icmlkeywords{Machine Learning, ICML}

  \vskip 0.3in
]



\printAffiliationsAndNotice{}  

\begin{abstract}
	The performance of Large Language Models (LLMs) is fundamentally influenced by the distributional composition of multi-domain pre-training data. While manual heuristics were prevalent in early models, they increasingly fail to capture the intricate synergies between domains as data complexity grows. To overcome the issue, a dominant approach seeks to fit a proxy function mapping between domain weights and their corresponding validation losses, and then find the optimal domain weights to minimize validation losses. These methods rely on strong structural assumptions, such as rank invariance or scaling laws, which are often violated, resulting in non-negligible estimation bias. A promising approach is to directly optimize the weighting scheme from data. However, it suffers from unstable optimization trajectory and prohibitive computational overhead, limiting its potential to search better domain weights configurations. This paper presents a Bayesian domain weighting method to infer the weights from a Dirichlet distribution via introducing Gamma prior information learned from observations. 	
    Experimental results demonstrate that proposed method could achieve stable and efficient domain weights learning, and identifies optimal mixtures while consuming substantially less data than search-based function-fitting methods, revitalizing optimization-based domain weighting for large-scale applications.
	
\end{abstract}

\section{Introduction}
\label{sec:intro}

%

\begin{figure}[t]
	\centering
	\begin{subfigure}[b]{0.23\textwidth}		
		\centering
		\includegraphics[width=\textwidth]{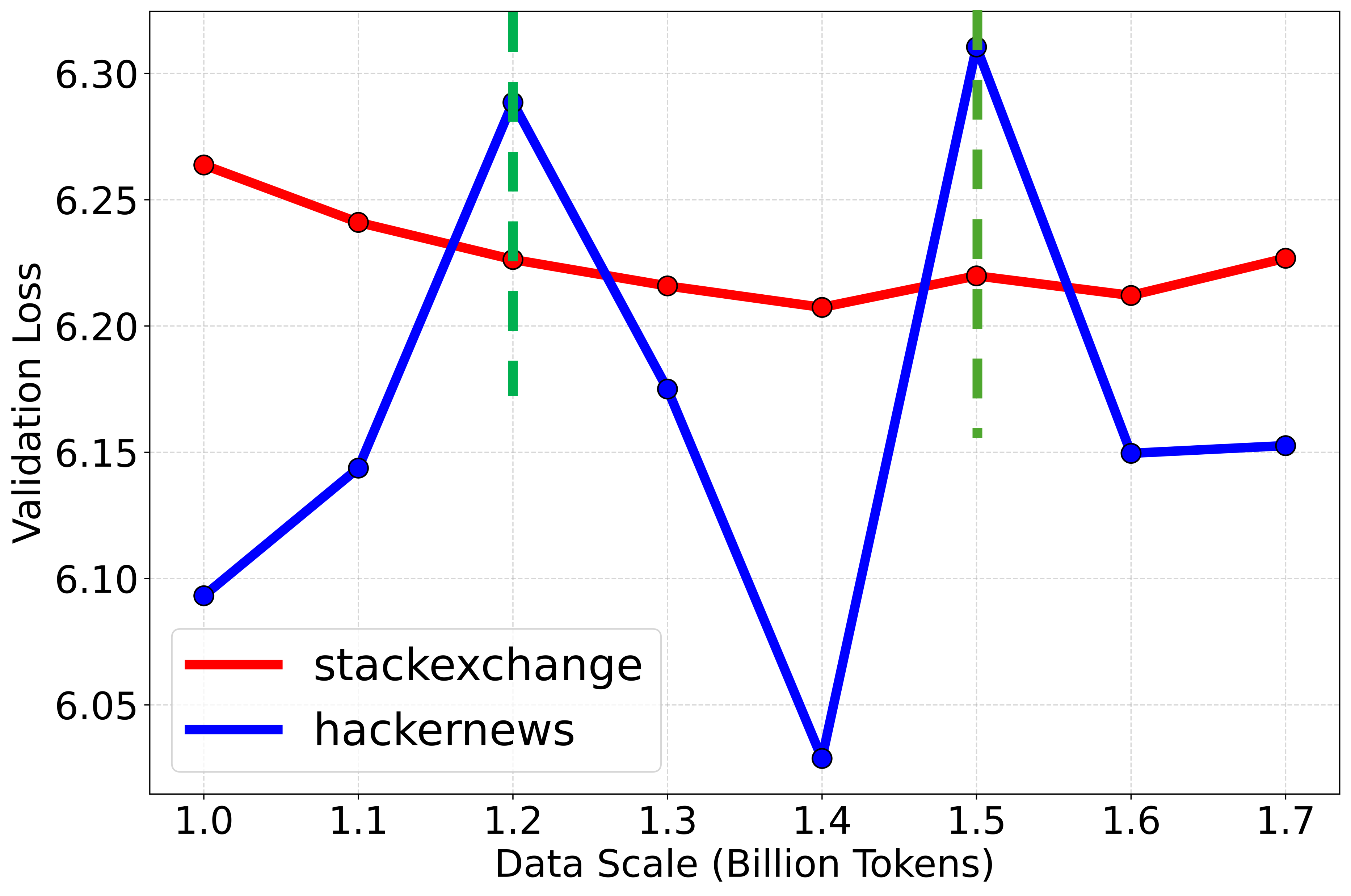}
		\caption{Rank at small scale.}
		\label{fig:rank_small}
	\end{subfigure}
	\hfill
	\begin{subfigure}[b]{0.23\textwidth}		
		\centering
		\includegraphics[width=\textwidth]{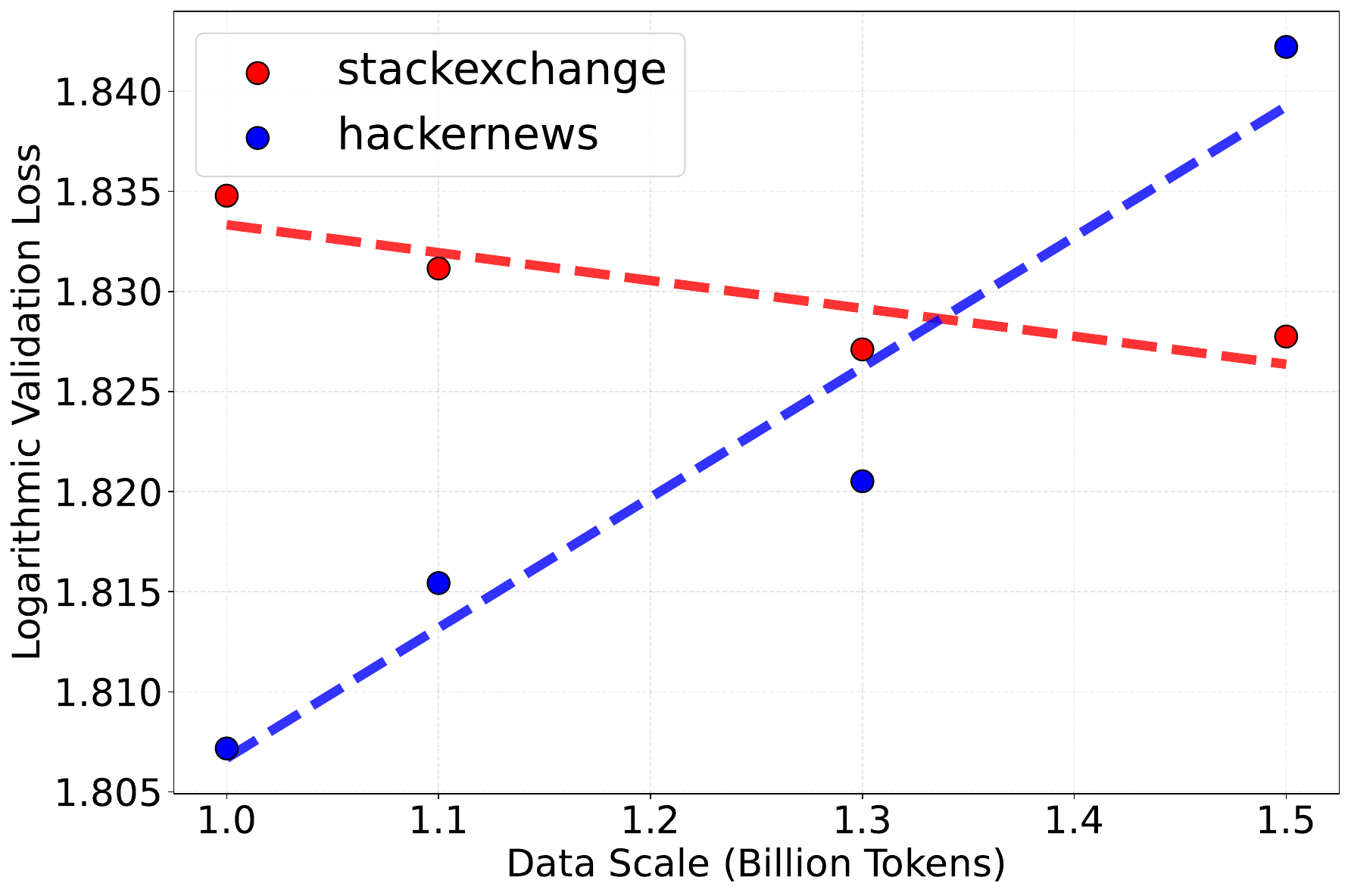}
		\caption{Fitting at small scale.}
		\label{fig:scaling_small}
	\end{subfigure}
	
	\begin{subfigure}[b]{0.23\textwidth}		
		\centering
		\includegraphics[width=\textwidth]{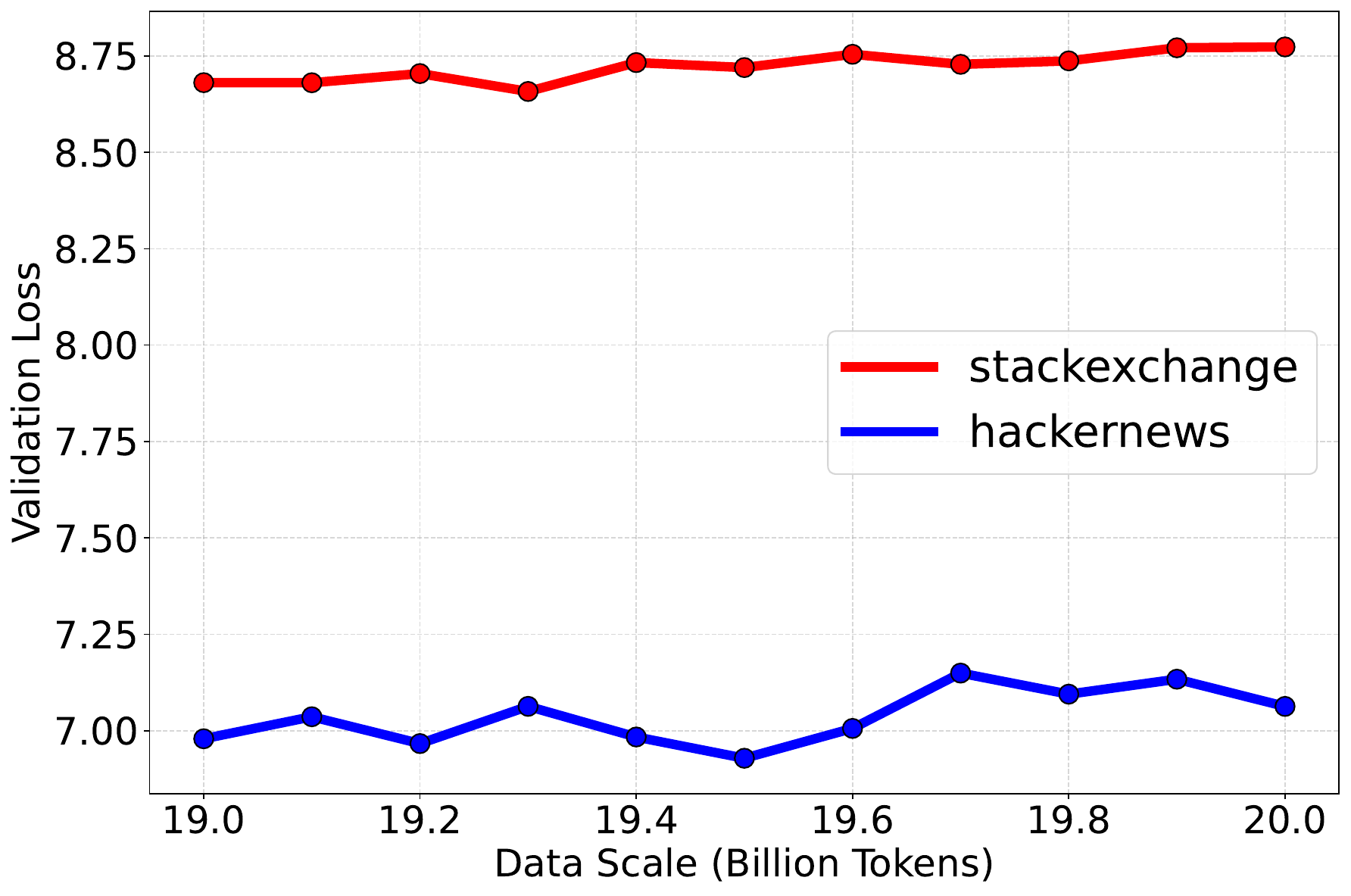}
		\caption{Rank at 10x scale.}
		\label{fig:rank_large}
	\end{subfigure}
	\hfill
	\begin{subfigure}[b]{0.23\textwidth}		
		\centering
		\includegraphics[width=\textwidth]{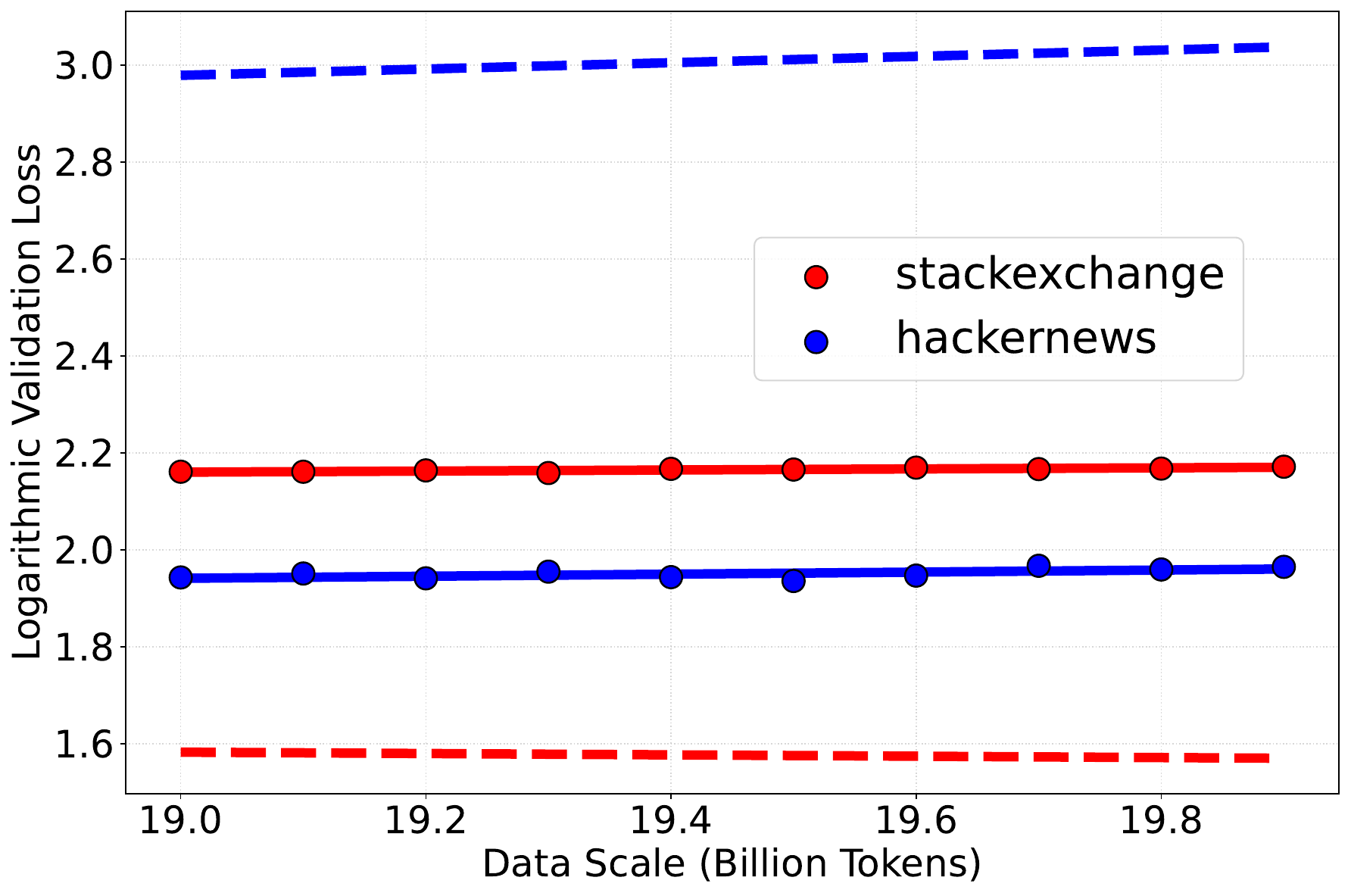}
		\caption{Prediction vs. Reality.}
		\label{fig:scaling_large}
	\end{subfigure}
	
	\caption{Empirical evidence of inconsistent domain importance shifts across data scales.
		Observations at limited data scales yield deceptive patterns that fail to generalize to the target-scale pre-training. 
		\textbf{Left Column (a, c)}: The green dashed lines highlight relative domain rankings observed at small scales; however, as shown in (c), these rankings are reversed at the 10x larger data scale, demonstrating a clear violation of the rank invariance assumption. 
		\textbf{Right Column (b, d)}: Dashed lines in (b) represent the power-law fits for the Red(stackexchange) and Blue(hackernews) domains derived from small-scale data. When extrapolated to the 10x target scale in (d), these predicted trends (dashed) diverge significantly and contradict the actual observed losses (solid), proving the substantial extrapolation bias in scaling law-based paradigms.}
	\label{fig:assumption_violations}
\end{figure}

\begin{figure*}[t]
	\centering
	\begin{subfigure}[b]{0.95\textwidth}		
		\centering
		\includegraphics[width=\textwidth]{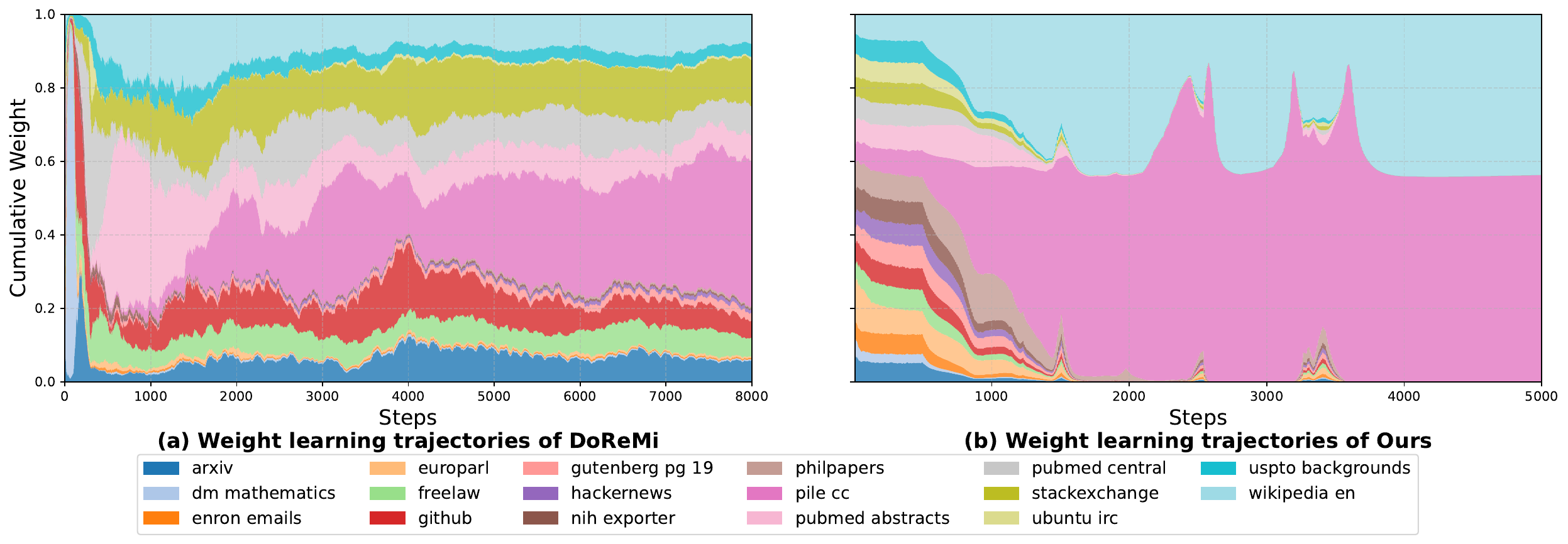} 
	\end{subfigure}	
	\caption{Comparison of domain weights learning trajectories. (a) DoReMi (Baseline) based on direct weights optimization: The trajectories exhibit high-variance oscillations and jitter, struggling to reach a stable convergence due to the inherent variance in domain-specific loss signals. (b) Our Method : By contrast, proposed ByDoRe method based on Bayesian inference on the domain weights
		yields much smoother transitions and achieves rapid, stable convergence, demonstrating superior robustness against data variance.}
	\label{fig:motivation} 
\end{figure*}

The performance of Large Language Models (LLMs) \cite{zhao2023survey, brown2020language} is significantly influenced by the distributional composition of multi-domain pre-training data. Recently, identifying the best data mixture, i.e., the mixing proportions of diverse source domains such as code, academic corpora, and web text, has emerged as a critical strategy for enhancing model capabilities \cite{gao2020pile, penedo2023refinedweb, radford2018improving, radford2019language, brown2020language, anil2023palm2, liang2025modomodo, zhuang2025meta, peng2025dataman}. Early researches, such as GPT-3 \cite{brown2020language} and Llama 2 \cite{touvron2023llama2}, primarily relied on manual heuristics, where mixture ratios were specified by human intuition. However, as data scales and domain complexities increase, these static, intuition-based configurations often fail to capture the intricate synergistic or conflicting relationships between domains, resulting in suboptimal performance \cite{albalak2024survey, liu2025rethinking, chen2024datajuicer, xie2023doremi, liu2024regmix, kang2024autoscale, ye2024datamixing, ge2024bimix, belenki2025mde, hoffmann2022chinchilla, gu2025data, shukor2025scaling}.

To overcome the limitations of manual tuning, recent studies \cite{albalak2024survey, liu2025rethinking, corrado2025automixalign} have transitioned toward automated domain weights learning. A dominant line of research treats domain weighting as a function-fitting task \cite{liu2024regmix, kang2024autoscale, ye2024datamixing, ge2024bimix, belenki2025mde, yang2025data}. These methods involve employing a function to fit the relationship between model performance (e.g., validation loss) and data mixtures, and then using the function to optimize the data mixture. To avoid the prohibitive cost of searching within massive data regimes, these methods introduce strong structural assumptions, such as rank invariance \cite{liu2024regmix, belenki2025mde} or scaling laws \cite{kang2024autoscale}, which effectively transforms the original optimization problem with notoriously expensive computation on the full-scale corpus into a cheaper one defined on a small subset of the data. 

Though they achieve promising results compared with manual strategies, the learned weighting scheme in small-data regimes would introduce substantial estimation bias when applied to the target full-scale pre-training.
As illustrated in Figure~\ref{fig:assumption_violations}, empirical evidence suggests that domain importance profiles often exhibit inconsistent shifts as the data volume increases. Existing methods like RegMix \cite{liu2024regmix}, AutoScale \cite{kang2024autoscale}, and MDE \cite{belenki2025mde} suffer from significant performance degradation due to this discrepancy (see Table~\ref{tab:case2_results}).

\begin{table*}[t]
	\centering
	\caption{Capabilities and requirements comparison of different data-mixing methods. Our method (ByDoRe) achieves superior performance across all scenarios with minimal training overhead and no restrictive assumptions.}
	\label{tab:method_comparison}
	\vskip 0.15in
	\begin{small}
		\begin{tabular}{lccccc}
			\toprule
			\multirow{2}{*}{Method} & \multirow{2}{*}{Stability} & \multirow{2}{*}{Assumptions} & \multirow{2}{*}{Cost} & \multicolumn{2}{c}{Performance} \\
			\cmidrule(r){5-6}
			& & & & Assumption Met & Assumption Not Met \\
			\midrule
			DoReMi~\cite{xie2023doremi}    & Unstable & None                   & Very High & Moderate & Moderate \\
			RegMix~\cite{liu2024regmix}    & Stable   & Rank Invariance & High      & Good     & Moderate \\
			AutoScale~\cite{kang2024autoscale} & Stable   & Scaling Law            & Low       & Moderate & Moderate \\
			MDE~\cite{belenki2025mde}      & Stable   & Rank Invariance & Low       & Moderate & Moderate \\
			\midrule
			\textbf{ByDoRe (Ours)} & \textbf{Stable} & \textbf{None} & \textbf{Very Low} & \textbf{Good} & \textbf{Good} \\
			\bottomrule
		\end{tabular}
	\end{small}
	\vskip -0.1in
\end{table*}

An alternative approach is to directly optimize the weighting scheme from data~\cite{xie2023doremi, panetal2025scalebio, held2025optimizing}. The early DoReMi method \cite{xie2023doremi} leverages distributionally robust optimization to tune the domain weights in an online learning manner. DoReMi takes the averaged domain weights over the whole training steps, while oscillates violently throughout the training process due to one-pass training process of LLM, whose weights learning trajectory is depicted in Figure~\ref{fig:motivation}(a). 
Such unstable weights learning process leads to slow convergence, often requiring a nearly exhaustive traversal of the trillion-token corpus to achieve a stable domain weights configuration, limiting its potential to search better domain weights configurations.
Besides, this excessive computational overhead has rendered direct optimization of weighting scheme practically unaffordable for practical LLM training.

To make the optimization of domain weights stable and efficient, we consider the weights learning from a probabilistic modeling perspective \cite{bishop2006pattern, gelman2013bayesian}. Specifically, we propose a Bayesian Domain Reweighting (ByDoRe) framework that formulates the domain weights as a Dirichlet distribution, where additional uncertainty helps eliminate the negative effect of deterministic weights learning \cite{wang2017robust}. Furthermore, we introduce the Gamma prior information to improve efficiency of inferring domain weights distribution. Inspired by the empirical Bayesian method \cite{efron2012large}, we propose a prior prediction network to adaptively determine the hyperparameters of Gamma distribution from data observations.
As empirically demonstrated by the weights learning trajectories in Figure~\ref{fig:motivation}(b), this data-driven prior predictor helps promote weights learning
more reliable and efficient.

To summarize, our contributions are threefold:

(1) We propose a Bayesian domain reweighting (ByDoRe) framework to address the issues of instability and slow convergence for optimizing domain weights methods. This method is helpful to yield a stable weights learning trajectory for language model pretraining.

(2) We introduce a Gamma-Dirichlet hierarchical bayesian model to help infer domain weights, and employ a prior prediction network to provide point estimation on the Gamma prior. It can flexibly adapt to domain weights oscillations in online one-pass learning without access to full-scale corpus, enhancing its practicality in real scenarios.

(3) ByDoRe achieves average general-purpose downstream accuracy of the SOTA RegMix method trained on the Pile only with 0.8\% training cost, and improve special target downstream accuracy of the SOTA RegMix method by 5.2\% points on zero-shot tasks only with 
1.2\% training cost. Table \ref{tab:method_comparison} compares ByDoRe with various data mixture methods, highlighting our method provides a
stable and agile framework for efficient as well as effective
data mixing for large-scale LLMs pretraining.

\section{Methods}

\subsection{Preliminaries}
\label{sec:preliminaries}

\textbf{Problem Setup.}
Consider a pre-training corpus
$\mathcal{D}=\bigcup_{k=1}^K D_k=\{x_n\}_{n=1}^N$
partitioned into $K$ domains. The data mixture is specified by a
domain-weight vector $\mathbf{w}\in\Delta^{K-1}$, where
$w_k\geq 0$ and $\sum_{k=1}^K w_k=1$. A training sequence is
sampled from the additive mixture $	p_{\mathbf{w}}(x) = \sum_{k=1}^K w_k\,p(x\mid D_k), \label{eq:mixture_distribution} $
where $p(x\mid D_k)$ denotes the uniform distribution over domain
$D_k$. Therefore, $p_{\mathbf{w}}(x)$ is a valid probability
distribution because it is a convex combination of valid
domain-conditional distributions. The expected training objective
under this mixture is
\[
\begin{aligned}
	\mathcal{L}_{\mathrm{tr}}(\theta,\mathbf{w})
	&:=
	\mathbb{E}_{x\sim p_{\mathbf{w}}(x)}
	[\ell_\theta(x)]
	\\
	&=
	\sum_{k=1}^K
	w_k\,
	\mathbb{E}_{x\sim p(x\mid D_k)}
	[\ell_\theta(x)],
\end{aligned}
\]
where $\ell_\theta(x)$ is the next-token prediction loss. The goal
of domain weighting is to identify a mixture $\mathbf{w}$ whose
resulting pre-trained model performs well on the target validation
distribution.

\textbf{DoReMi.}
A principled baseline for directly optimizing domain weights is
\textit{DoReMi}~\cite{xie2023doremi}. The domain-weight vector
$\mathbf{w}$ is obtained by minimizing the worst-case excess loss
across domains through distributionally robust optimization:
\begin{equation}
	\min_{\theta}
	\max_{\mathbf{w}\in\Delta^{K-1}}
	\sum_{i=1}^K
	w_i
	\left[
	\frac{
		\sum_{x\in D_i}
		\bigl(\ell_\theta(x)-\ell_{\mathrm{ref}}(x)\bigr)
	}{
		\sum_{x\in D_i}|x|
	}
	\right],
	\label{eq:doremi}
\end{equation}
where $\ell_{\mathrm{ref}}(x)$ is the loss of a reference model.
In this framework, $\mathbf{w}$ is treated as a deterministic
vector. The optimizer produces a sequence
$\{\mathbf{w}_t\}_{t=1}^T$ through exponential gradient descent
based on the observed excess losses, and the final domain weights
are obtained by averaging the optimization trajectory:
\[
\bar{\mathbf{w}}
=
\frac{1}{T}
\sum_{t=1}^T \mathbf{w}_t.
\]
Unlike conventional optimization settings in which training
samples may be revisited over multiple epochs, large-scale LLM
pre-training typically processes each sample only once. Under this
one-pass setting, stochastic domain-specific loss signals can
induce substantial oscillations in the learned weights, as shown
in Figure~\ref{fig:motivation}(a). Consequently, direct weight
optimization may require traversing a large fraction of the
pre-training corpus before producing a stable mixture, resulting
in prohibitive computational cost.

\subsection{Bayesian Domain Reweighting Framework}
\label{sec:bydore}

To stabilize domain-weight optimization, we reformulate domain
weighting from a probabilistic perspective. Instead of directly
optimizing a deterministic weight vector, we infer a probability
distribution over domain mixtures. Modeling uncertainty in this
way reduces the sensitivity of the learned mixture to stochastic
batch-level loss fluctuations.

\begin{figure}[tbp]
	\centering
	\begin{subfigure}[b]{0.48\textwidth}
		\centering
		\includegraphics[width=\textwidth]{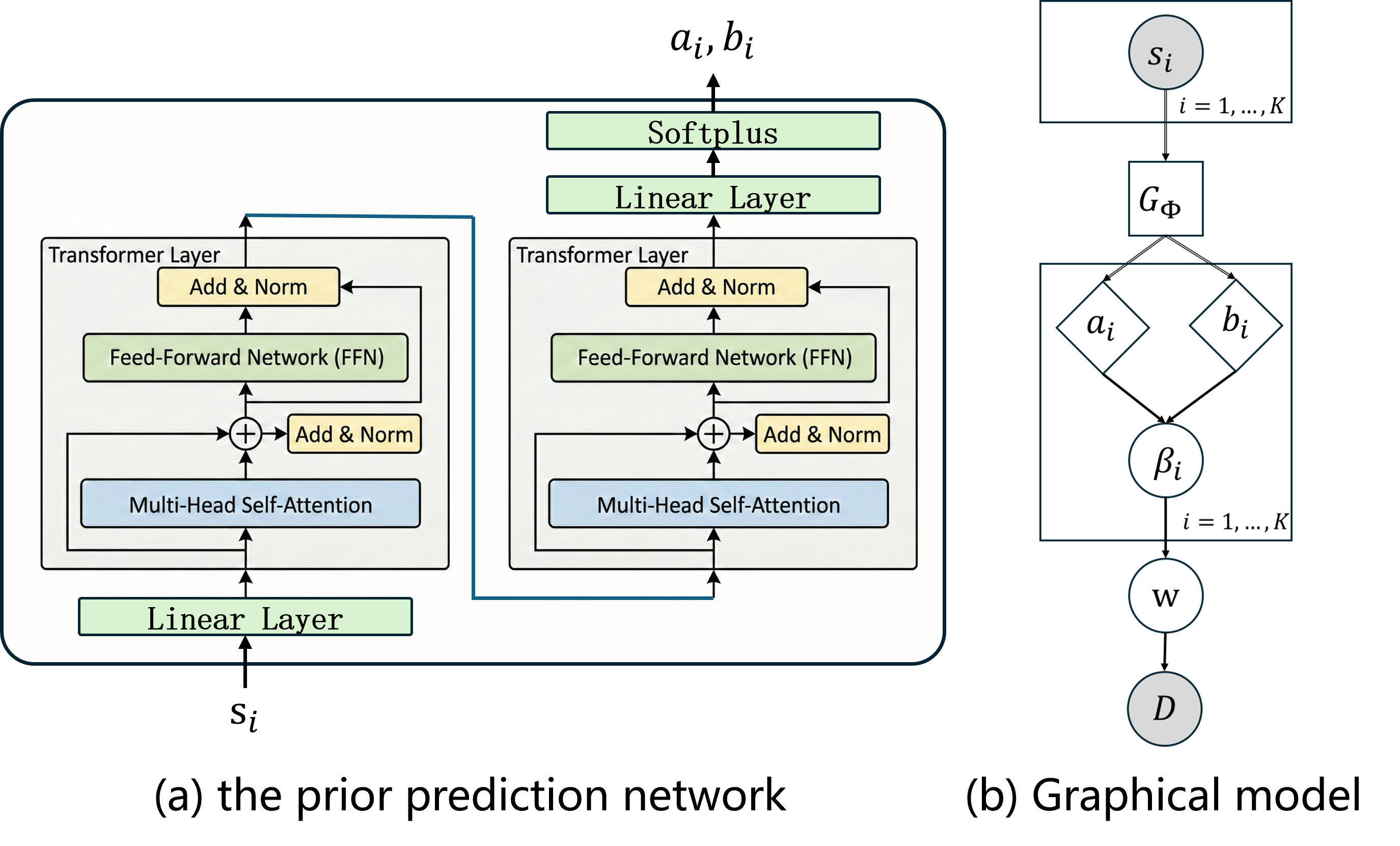}
		\label{fig:arch}
	\end{subfigure}
	\caption{
		\textbf{Overview of the hierarchical Bayesian domain
			reweighting framework.}
		Panel (a) shows the prior prediction network
		$\mathcal{G}_{\Phi}$, which maps empirical domain signals
		$\mathbf{s}$ to the hyperparameters of the Gamma
		distributions. Panel (b) illustrates the generative
		process in which the resulting concentration variables
		govern the latent domain weights used to construct the
		pre-training mixture.
	}
	\label{fig:framework_overview}
\end{figure}

We treat the domain weights $\mathbf{w}$ as latent random variables
with prior $p(\mathbf{w})$. For a given mixture, define the
per-sample weight as
\[
\pi_{\mathbf{w}}(x_n)
:=
\frac{w_k}{|D_k|},
\qquad
\forall x_n\in D_k,\quad k=1,\ldots,K.
\]
Thus, all samples belonging to the same domain receive the same
per-sample weight, while the total contribution of domain $D_k$ is
exactly $w_k$. Following Bayesian data
reweighting~\cite{wang2017robust}, we define the corresponding
reweighted joint model as
\[
p(\mathcal{D},\theta,\mathbf{w})
\propto
p(\theta)\,p(\mathbf{w})
\prod_{n=1}^N
p_\theta(x_n)^{\pi_{\mathbf{w}}(x_n)},
\]
where $p_\theta(x_n)$ denotes the autoregressive likelihood of
$x_n$. The resulting generalized likelihood used for parameter
estimation is
\[
p(\mathcal{D}\mid\theta,\mathbf{w})
\propto
\prod_{n=1}^N
p_\theta(x_n)^{\pi_{\mathbf{w}}(x_n)}.
\]
This reweighted likelihood is conceptually distinct from the
additive sampling distribution in
Eq.~(\ref{eq:mixture_distribution}): the former defines the
estimation objective, whereas the latter specifies how training
samples are drawn.

Our goal is to infer domain weights that yield a pre-trained model
with strong performance on a validation set
$\mathcal{D}_{\mathrm{val}}$. The corresponding posterior
predictive distribution is
\begin{align}
	p(
	\mathcal{D}_{\mathrm{val}}
	\mid
	\mathbf{w},\mathcal{D}
	)
	=
	\int
	p(\mathcal{D}_{\mathrm{val}}\mid\theta)
	p(\theta\mid\mathbf{w},\mathcal{D})
	\,d\theta.
	\label{eqw}
\end{align}
We approximate this posterior predictive distribution using a
plug-in estimator:
\begin{align}
	p(
	\mathcal{D}_{\mathrm{val}}
	\mid
	\mathbf{w},\mathcal{D}
	)
	\approx
	p(
	\mathcal{D}_{\mathrm{val}}
	\mid
	\hat{\theta}(\mathbf{w})
	).
\end{align}
Using the standard maximum-likelihood pre-training objective,
equivalently a flat prior over $\theta$, the plug-in estimator is
obtained from the reweighted likelihood. Because
$\ell_\theta(x)=-\log p_\theta(x)$, we have
\begin{align}
	\hat{\theta}(\mathbf{w})
	&=
	\arg\max_{\theta}
	\log p(\mathcal{D}\mid\theta,\mathbf{w})
	\notag\\
	&=
	\arg\max_{\theta}
	\sum_{n=1}^N
	\pi_{\mathbf{w}}(x_n)
	\log p_\theta(x_n)
	\notag\\
	&=
	\arg\max_{\theta}
	\sum_{k=1}^K
	\frac{w_k}{|D_k|}
	\sum_{x\in D_k}
	\log p_\theta(x)
	\notag\\
	&=
	\arg\max_{\theta}
	\sum_{k=1}^K
	w_k
	\mathbb{E}_{x\sim p(x\mid D_k)}
	[\log p_\theta(x)]
	\notag\\
	&=
	\arg\max_{\theta}
	\mathbb{E}_{x\sim p_{\mathbf{w}}(x)}
	[\log p_\theta(x)]
	\notag\\
	&=
	\arg\min_{\theta}
	\mathcal{L}_{\mathrm{tr}}(\theta,\mathbf{w}).
	\label{Eqtheta}
\end{align}
The fourth equality follows from the uniformity of
$p(x\mid D_k)$, and the fifth follows from the linearity of
expectation under the additive mixture in
Eq.~(\ref{eq:mixture_distribution}).

The posterior distribution of $\mathbf{w}$ can then be written as
\begin{align*}
	p(
	\mathbf{w}
	\mid
	\mathcal{D},\mathcal{D}_{\mathrm{val}}
	)
	&\propto
	p(
	\mathcal{D}_{\mathrm{val}}
	\mid
	\mathbf{w},\mathcal{D}
	)
	p(\mathbf{w})
	\\
	&\approx
	p(
	\mathcal{D}_{\mathrm{val}}
	\mid
	\hat{\theta}(\mathbf{w})
	)
	p(\mathbf{w}).
\end{align*}
Directly evaluating this posterior is intractable. We therefore
introduce a variational distribution $q_{\Phi}(\mathbf{w})$ and
derive the following evidence lower bound:
\begin{align*}
	&\log
	p(
	\mathcal{D}_{\mathrm{val}}
	\mid
	\mathcal{D}
	)
	\\
	&\quad\approx
	\log
	\int
	p(
	\mathcal{D}_{\mathrm{val}}
	\mid
	\hat{\theta}(\mathbf{w})
	)
	p(\mathbf{w})
	\,d\mathbf{w}
	\\
	&\quad=
	\log
	\int
	q_{\Phi}(\mathbf{w})
	\frac{
		p(
		\mathcal{D}_{\mathrm{val}}
		\mid
		\hat{\theta}(\mathbf{w})
		)
		p(\mathbf{w})
	}{
		q_{\Phi}(\mathbf{w})
	}
	\,d\mathbf{w}
	\\
	&\quad\geq
	\mathbb{E}_{q_{\Phi}(\mathbf{w})}
	\left[
	\log
	p(
	\mathcal{D}_{\mathrm{val}}
	\mid
	\hat{\theta}(\mathbf{w})
	)
	\right]
	-
	\mathrm{KL}
	\left(
	q_{\Phi}(\mathbf{w})
	\,\|\,p(\mathbf{w})
	\right)
	\\
	&\quad\triangleq
	\mathcal{L}(q_{\Phi}).
\end{align*}
The first term encourages mixtures that yield strong validation
performance, while the KL term regularizes the inferred
distribution toward the prior $p(\mathbf{w})$. We learn
$q_{\Phi}(\mathbf{w})$ by solving
\begin{equation}
	\Phi^{*}
	=
	\arg\max_{\Phi}
	\mathcal{L}(q_{\Phi}).
\end{equation}

\subsection{Formulation of \texorpdfstring{$q_{\Phi}(\mathbf{w})$}
	{q-Phi(w)}}
\label{sec:godam}

We parameterize the variational distribution using a
Gamma--Dirichlet hierarchical model:
\begin{equation}
	q_{\Phi}(\mathbf{w})
	=
	\int
	\operatorname{Dir}
	(\mathbf{w}\mid\boldsymbol{\beta})
	\operatorname{Gamma}
	\left(
	\boldsymbol{\beta}
	\mid
	\mathbf{a}_{\Phi}(\mathbf{s}),
	\mathbf{b}_{\Phi}(\mathbf{s})
	\right)
	d\boldsymbol{\beta}.
	\label{sample}
\end{equation}
Here, $\operatorname{Dir}
(\mathbf{w}\mid\boldsymbol{\beta})$ is a Dirichlet distribution
governed by the positive concentration parameters
$\boldsymbol{\beta}$. These concentration parameters follow Gamma
distributions whose shape and rate parameters,
$\mathbf{a}_{\Phi}(\mathbf{s})$ and
$\mathbf{b}_{\Phi}(\mathbf{s})$, depend on the observed
domain-specific signals $\mathbf{s}$.

The process consists of the following three steps.

\textbf{(i) Prior-parameter prediction.}
The prior prediction network generates
\[
\bigl(
\mathbf{a}_{\Phi}(\mathbf{s}),
\mathbf{b}_{\Phi}(\mathbf{s})
\bigr)
=
\mathcal{G}_{\Phi}(\mathbf{s}).
\]
As shown in Figure~\ref{fig:framework_overview}(a),
$\mathcal{G}_{\Phi}$ consists of an initial linear projection, a
two-layer Transformer encoder, and an output linear layer followed
by a Softplus activation that ensures positive Gamma parameters.
At iteration $t$, its input is the vector of empirical
domain-specific losses
\[
\mathbf{s}_t
=
\left[
\ell_{\theta_t}(b_{t,1}),
\ldots,
\ell_{\theta_t}(b_{t,K})
\right],
\qquad
b_{t,k}\sim D_k,
\]
where $\ell_{\theta_t}(b_{t,k})$ is the average next-token loss on
a small batch sampled from domain $D_k$.

\textbf{(ii) Concentration-variable sampling.}
The Dirichlet concentration variables are sampled as
\[
\boldsymbol{\beta}
\sim
\operatorname{Gamma}
\left(
\mathbf{a}_{\Phi}(\mathbf{s}),
\mathbf{b}_{\Phi}(\mathbf{s})
\right).
\]
The Gamma distribution provides positive, scale-flexible
concentration parameters.

\textbf{(iii) Domain-weight sampling.}
The domain weights are subsequently sampled as
\[
\mathbf{w}
\sim
\operatorname{Dir}
(\boldsymbol{\beta}).
\]
The Dirichlet distribution places the weights on the probability
simplex and therefore enforces
$w_k\geq 0$ and $\sum_{k=1}^K w_k=1$.

\subsection{Learning Algorithm}

\textbf{Main Model Update.}
At iteration $t$, we first sample
$\mathbf{w}_t\sim q_{\Phi}(\mathbf{w})$ and construct a training
batch $B_t$ according to the additive mixture
$p_{\mathbf{w}_t}$ in Eq.~(\ref{eq:mixture_distribution}). The
model parameters are then updated by
{\small
	\begin{equation}
		B_t
		\sim
		p_{\mathbf{w}_t}^{\otimes |B_t|},
		\qquad
		\theta_{t+1}
		=
		\theta_t
		-
		\eta
		\nabla_{\theta}
		\mathcal{L}_{\mathrm{tr}}(\theta_t;B_t).
		\label{eq:theta_update_iter}
	\end{equation}
}

\textbf{Domain-Weight Distribution Update.}
We update the prior prediction network by maximizing the ELBO:
\begin{equation}
	\Phi_{t+1}
	=
	\Phi_t
	+
	\gamma
	\nabla_{\Phi}
	\mathcal{L}(q_{\Phi}).
\end{equation}
We use a uniform distribution on the simplex as
$p(\mathbf{w})$ and employ implicit reparameterization
gradients~\cite{figurnov2018implicit} to optimize
$q_{\Phi}(\mathbf{w})$.

Computing the exact hypergradient
\[
\nabla_{\Phi}
\log
p(
\mathcal{D}_{\mathrm{val}}
\mid
\hat{\theta}(\mathbf{w})
)
\]
requires differentiating through the training trajectory and
therefore involves prohibitively expensive Hessian--vector
products. We instead employ a central finite-difference
approximation, whose detailed derivation is provided in
Appendix~\ref{app:first_order}. Let
\[
\Delta
=
\nabla_{\theta}
\log
p(
\mathcal{D}_{\mathrm{val}}
\mid
\hat{\theta}(\mathbf{w})
).
\]
The update of $\Phi$ is approximated by
{\small
	\begin{equation}
		\begin{aligned}
			\Phi_{t+1}
			=
			\Phi_t
			-
			&\gamma
			\Bigg(
			\frac{\eta}{2\epsilon}
			\Big[
			\nabla_{\Phi}
			\mathcal{L}_{\mathrm{tr}}
			(
			\theta_t+\epsilon\Delta,
			\bar{\mathbf{w}}_t
			) -
			\\
			&\nabla_{\Phi}
			\mathcal{L}_{\mathrm{tr}}
			(
			\theta_t-\epsilon\Delta,
			\bar{\mathbf{w}}_t
			)
			\Big]
			+
			\lambda
			\nabla_{\Phi}
			\mathcal{R}_{\mathrm{KL}}(\Phi_t)
			\Bigg),
		\end{aligned}
		\label{eq:phi_update_iter}
	\end{equation}
}
where $\epsilon>0$ is the finite-difference step size and
$\mathcal{R}_{\mathrm{KL}}$ denotes the KL-regularization term in
the ELBO.

We use the expectation of domain-weight distribution as a
differentiable surrogate during the meta-update:
\[
\bar{\mathbf{w}}_t
=
\mathbb{E}_{q_{\Phi_t}(\mathbf{w})}
[\mathbf{w}].
\]
Its $k$-th component is computed as
\begin{equation}
	\bar{w}_{t,k}
	=
	\frac{
		a_{t,k}/b_{t,k}
	}{
		\sum_{j=1}^K a_{t,j}/b_{t,j}
	},
	\qquad
	k=1,\ldots,K.
	\label{eq:expectation_w}
\end{equation}
The final mixture is given by the resulting expected weight vector
$\bar{\mathbf{w}}$.

\begin{algorithm}[tb]
	\caption{Bayesian Domain Reweighting (ByDoRe)}
	\label{alg:ByDoRe}
	\begin{algorithmic}[1]
		\STATE \textbf{Input:}
		Domains $\{D_k\}_{k=1}^K$, proxy model $\theta_0$,
		prior prediction network $\Phi_0$, and validation set
		$\mathcal{D}_{\mathrm{val}}$.
		\WHILE{not converged}
		\STATE Sample
		$\{b_{t,k}\sim D_k\}_{k=1}^K$ and compute
		$\mathbf{s}_t=
		[\ell_{\theta_t}(b_{t,k})]_{k=1}^K$.
		\STATE Compute
		$(\mathbf{a}_t,\mathbf{b}_t)
		=
		\mathcal{G}_{\Phi_t}(\mathbf{s}_t)$.
		\STATE Sample
		$\boldsymbol{\beta}_t
		\sim
		\operatorname{Gamma}
		(\mathbf{a}_t,\mathbf{b}_t)$.
		\STATE Sample
		$\mathbf{w}_t
		\sim
		\operatorname{Dir}
		(\boldsymbol{\beta}_t)$.
		\STATE Construct $B_t$ according to
		$p_{\mathbf{w}_t}$ in
		Eq.~(\ref{eq:mixture_distribution}).
		\STATE Update $\theta_{t+1}$ using
		Eq.~(\ref{eq:theta_update_iter}).
		\STATE Compute
		$\Delta
		=
		\nabla_{\theta}
		\log
		p(
		\mathcal{D}_{\mathrm{val}}
		\mid
		\theta_{t+1}
		)$.
		\STATE Compute $\bar{\mathbf{w}}_t$ using
		Eq.~(\ref{eq:expectation_w}).
		\STATE Update $\Phi_{t+1}$ using
		Eq.~(\ref{eq:phi_update_iter}).
		\ENDWHILE
		\STATE \textbf{Output:}
		Expected domain weights $\bar{\mathbf{w}}$.
	\end{algorithmic}
\end{algorithm}

\subsection{Theoretical Guarantees}

We establish convergence guarantees for both the variational
meta-objective and the proxy-model training objective. Under
standard smoothness and boundedness conditions on
$\mathcal{L}_{\mathrm{tr}}(\theta,\mathbf{w})$, the validation
objective, and the variational distribution
$q_{\Phi}(\mathbf{w})$, the alternating updates in
Algorithm~\ref{alg:ByDoRe} satisfy the following results.

\begin{theorem}[Meta-Objective Convergence]
	Let $\{\theta_t,\Phi_t\}_{t=0}^{T}$ be the iterates generated by
	Algorithm~\ref{alg:ByDoRe}, where the update of $\Phi$ uses the
	finite-difference approximation and the KL-regularization term in
	Eq.~(\ref{eq:phi_update_iter}). With an appropriate choice of step
	sizes, the ELBO satisfies
	\[
	\min_{0\leq t\leq T-1}
	\mathbb{E}
	\left[
	\left\|
	\nabla_{\Phi}
	\mathcal{L}(q_{\Phi_t})
	\right\|^2
	\right]
	\leq
	\mathcal{O}
	\left(
	\frac{1}{\sqrt{T}}
	\right)
	+
	\mathcal{O}(\epsilon^4),
	\]
	where $\epsilon$ is the finite-difference step size in
	Eq.~(\ref{eq:phi_update_iter}).
\end{theorem}

\begin{theorem}[Training Objective Convergence]
	Under the same conditions, the model parameters updated according
	to Eq.~(\ref{eq:theta_update_iter}) satisfy
	\[
	\min_{0\leq t\leq T-1}
	\mathbb{E}
	\left[
	\left\|
	\nabla_{\theta}
	\mathcal{L}_{\mathrm{tr}}
	(
	\theta_t,
	\bar{\mathbf{w}}_t
	)
	\right\|^2
	\right]
	\leq
	\mathcal{O}
	\left(
	\frac{1}{\sqrt{T}}
	\right)
	+
	\mathcal{O}(\epsilon^2),
	\]
	where
	$\bar{\mathbf{w}}_t
	=
	\mathbb{E}_{q_{\Phi_t}}[\mathbf{w}]$
	is given by Eq.~(\ref{eq:expectation_w}).
\end{theorem}

The two theorems establish convergence to stationary points at an
$\mathcal{O}(1/\sqrt{T})$ rate, up to the approximation errors
introduced by the finite-difference estimator. Detailed
assumptions and proofs are provided in
Appendix~\ref{app:DPT1}.

\section{Experimental Results}
\label{sec:experiments}

\begin{table*}[t]
	\centering
	\caption{Zero-shot performance comparison on 1B target models (150B tokens) for Case 1. Best automated results are in \textbf{bold}.}
	\label{tab:case1_results}
	\small
	\begin{tabular}{l c c c c c c c}
		\toprule
		\textbf{Task} & \textbf{Human} & \textbf{Pile-CC} & \textbf{Doremi} & \textbf{RegMix} & \textbf{AutoScale} & \textbf{MDE} & \textbf{Ours} \\
		\midrule
		ARC-Easy    & 49.14 & 49.98 & 51.05 & 50.28 & 46.94 & 47.60 & \textbf{51.32} \\
		COPA        & 66.17 & 67.17 & 67.17 & 69.00 & 68.00 & 63.33 & \textbf{70.17} \\
		HellaSwag   & 37.93 & 42.26 & 41.88 & \textbf{43.65} & 37.99 & 37.63 & 41.58 \\
		LAMBADA     & 28.06 & 35.99 & \textbf{34.12} & 34.35 & 27.56 & 26.30 & 31.49 \\
		LogiQA      & 24.37 & 26.47 & 27.27 & 27.21 & 26.22 & 26.70 & \textbf{28.11} \\
		MultiRC     & 55.08 & 52.53 & 52.76 & 52.94 & 53.00 & 50.30 & \textbf{53.43} \\
		OpenBookQA  & 28.30 & 30.37 & 30.13 & \textbf{29.33} & 27.67 & 29.00 & 29.07 \\
		PIQA        & 64.91 & 68.91 & 68.06 & \textbf{69.12} & 65.29 & 66.23 & 68.03 \\
		QQP         & 38.99 & 43.77 & 46.92 & \textbf{50.19} & 49.41 & 46.24 & 49.08 \\
		RACE        & 30.88 & 31.24 & 30.69 & 31.18 & 29.44 & 30.40 & \textbf{31.33} \\
		SciQ        & 80.92 & 80.57 & 80.83 & 78.35 & 77.17 & 77.40 & \textbf{80.93} \\
		Social IQA  & 38.79 & 40.11 & \textbf{40.56} & 40.02 & 38.72 & 38.05 & 39.63 \\
		WinoGrande  & 51.16 & 52.38 & 51.80 & 51.24 & \textbf{52.58} & 51.87 & 51.68 \\
		\midrule
		\textbf{Overall Avg} & 45.75 & 47.83 & 47.94 & 48.22 & 46.15 & 45.47 & \textbf{48.50} \\
		\textbf{FLOPs cost}  & - & - & $2.10 \times 10^{18}$ & $3.07 \times 10^{18}$ & $2.55 \times 10^{17}$ & $1.02\times 10^{17}$ & \boldmath $2.40 \times 10^{16}$ \\
		\bottomrule
	\end{tabular}
\end{table*}
In this section, we evaluate the performance of the domain weights identified by ByDoRe. Our primary objective is to verify whether the weighting configurations derived from proposed method lead to superior performance when applied to large-scale target model training.

\textbf{Datasets and Scenarios.} We conduct our evaluation on \textit{The Pile} \cite{gao2020pile} using its 17 publicly available domains. To assess robustness across different data distributions, we design two primary scenarios: (1) \textbf{Case 1 (General-Purpose)}, which utilizes the full mixture of 17 domains to optimize for broad language modeling capabilities; and (2) \textbf{Case 2 (Specialized Targeting)}, where the target distribution is restricted to GitHub and Enron Emails. These two domains possess statistical properties that differ significantly from the general corpus, posing a rigorous challenge to existing function-fitting methods that rely on rank invariance or scaling law stability. For both scenarios, validation sets are sampled from \textit{The Pile CC}. Detailed descriptions of the dataset can be found in the appendix~\ref{app:dataset_details}.

\textbf{Baselines.} We compare ByDoRe with several state-of-the-art data scheduling strategies, including: \textbf{Human} manual weights; \textbf{DoReMi} \cite{xie2023doremi}, which employs minimax optimization; \textbf{RegMix} \cite{liu2024regmix}, which assumes rank invariance; \textbf{AutoScale} \cite{kang2024autoscale}, based on scaling law extrapolation; and \textbf{MDE} \cite{belenki2025mde}, utilizing expert model cooperation.

\textbf{Execution of the Proxy-Target Paradigm.} The evaluation follows a two-stage process. First, in the Search Phase, we utilize a 1M-parameter proxy LLMs model to identify the domain weights. For ByDoRe, we optimize the prior prediction network parameter $\Phi$ over 0.4B tokens, setting the finite difference scalar $\epsilon = 0.01 / \|\theta_t\|_2$ to compute meta-gradients. For baselines, we follow their configurations as detailed in their papers. Second, in the Verification Phase, the fixed domain weights identified by each method are used to train a 1B-parameter target model from scratch. All target models are trained on 150B tokens to ensure a fair and comprehensive comparison of the different data mixtures.

\textbf{Downstream Evaluation Benchmarks.} To assess the quality of the trained target models, we evaluate them across 13 diverse downstream benchmarks using the suite from \citet{liu2024regmix}. These include commonsense reasoning (e.g., Hellaswag \cite{zellers2019hellaswag}, PiQA \cite{bisk2020piqa}, WinoGrande \cite{sakaguchi2021winogrande}), reading comprehension (e.g., RACE \cite{lai2017race}, MultiRC \cite{khashabi2018multirc}), and specialized knowledge (e.g., SciQ \cite{welbl2017crowdsourcing}). This multi-dimensional evaluation ensures that the discovered weights are not merely overfitting the validation loss but are translating into improved model capabilities. Detailed descriptions of the validation dataset can be found in the appendix~\ref{app:validation_datasets}.

\begin{table*}[t]
	\centering
	\caption{Zero-shot performance comparison for Case 2. Best automated results are in \textbf{bold}.}
	\label{tab:case2_results}
	\small
	\begin{tabular}{l c c c c c c}
		\toprule
		\textbf{Task} & \textbf{Human} & \textbf{Doremi} & \textbf{RegMix} & \textbf{AutoScale} & \textbf{MDE} & \textbf{Ours} \\
		\midrule
		ARC-Easy    & 30.74 & 31.33 & 27.38 & 29.69 & 28.79 & \textbf{32.46} \\
		COPA        & 56.50 & 53.67 & 53.00 & 50.17 & \textbf{57.17} & 53.50 \\
		HellaSwag   & 27.46 & 28.68 & 25.77 & 26.95 & 26.91 & \textbf{28.73} \\
		LAMBADA     & 5.50 & \textbf{9.07} & 0.07 & 3.12 & 2.56 & 8.30 \\
		LogiQA      & 26.09 & 26.75 & 24.01 & 24.32 & 23.40 & \textbf{28.14} \\
		MultiRC     & 56.20 & 54.37 & 56.51 & 51.10 & 55.48 & \textbf{56.87} \\
		OpenBookQA  & 23.53 & \textbf{25.07} & 24.87 & 24.67 & 24.50 & 24.40 \\
		PIQA        & 54.23 & \textbf{55.62} & 51.91 & 53.92 & 53.19 & 54.61 \\
		QQP         & \textbf{46.56} & 38.80 & 37.30 & 37.33 & 37.76 & \textbf{39.00} \\
		RACE        & 22.98 & \textbf{24.91} & 21.48 & 23.98 & 23.41 & 24.58 \\
		SciQ        & 56.27 & \textbf{64.58} & 25.07 & 51.75 & 48.72 & 62.77 \\
		Social IQA  & 34.37 & 34.66 & 33.30 & 33.19 & 33.53 & \textbf{35.39} \\
		WinoGrande  & 49.93 & 48.75 & \textbf{50.26} & 48.76 & 48.88 & 49.76 \\
		\midrule
		\textbf{Overall Avg} & 37.72 & 38.18 & 33.15 & 35.45 & 35.72 & \textbf{38.35} \\
		\textbf{FLOPs cost}  & - & $1.68 \times 10^{17}$ & $1.22 \times 10^{18}$ & $1.02 \times 10^{17}$ & $4.08\times 10^{16}$ & \boldmath $1.44 \times 10^{16}$ \\
		\bottomrule
	\end{tabular}
\end{table*}

\subsection{Case 1: General-Purpose Language Modeling}
\label{sec:case1_results}

The zero-shot results for Case 1 are summarized in Table~\ref{tab:case1_results}. ByDoRe achieves an overall average score of $48.50\%$, outperforming automated baselines such as RegMix ($48.22\%$) and AutoScale ($46.15\%$). By effectively modeling the prior-assignment laws, our optimization-based approach fully manifests its performance advantage, regularizing the search against high-variance data signals.
A key observation lies in the identified weights distribution: ByDoRe is the only method that correctly identifies Wikipedia (en) as a primary knowledge source, assigning it a weight of $0.31$, which is balanced against the $0.59$ weight for Pile-CC. In contrast, existing automated methods tend to overfit the validation set distribution; RegMix and DoReMi assign disproportionately high weight to Pile-CC ($0.87$ and $0.61$, respectively) while neglecting Wikipedia ($0.02$ and $0.07$, respectively). Although DoReMi recognizes Wikipedia as the second most important domain, it remains biased towards the validation distribution. This balanced prioritization allows ByDoRe to achieve superior performance.

Furthermore, ByDoRe exhibits exceptional computational efficiency during the weights identification phase, delivering a superior trade-off between speed and accuracy compared to established baselines. While RegMix achieves the highest accuracy among existing automated methods at 48.22\%, ByDoRe surpasses it with an overall average of 48.50\% while requiring only $2.40 \times 10^{16}$ FLOPs, a reduction in search overhead of over 127 times compared to RegMix's $3.07 \times 10^{18}$ FLOPs (less than 0.8\% training cost of RegMix). Even against the fastest baseline MDE, our method is over 4 times faster while improving accuracy by 3.03 percentage points. The robustness of our Bayesian domain weighting framework is even more pronounced compared to direct optimization methods like DoReMi, which requires $2.10 \times 10^{18}$ FLOPs, over 87 times our budget, yet exhibits high-variance oscillations and jitter in its weight learning trajectories. In contrast, ByDoRe yields smoother transitions and achieves rapid, stable convergence within fewer steps, confirming that our method effectively minimizes computational requirements while surpassing SOTA methods.

\subsection{Case 2: Specialized Targeting and Robustness}
\label{sec:case2_results}

The results for Case 2 (Table~\ref{tab:case2_results}) reveal a significant performance collapse for function-fitting methods such as RegMix, AutoScale, and MDE. In this scenario, the dataset Enron Emails violate standard assumptions of rank invariance and scaling law stability (see Figure~\ref{fig:appendix_loss}). Consequently, these methods are significantly outperformed by human heuristics, proving that their reliance on rigid fitting models lacks the necessary flexibility for specialized distributions.
In contrast, ByDoRe effectively bridges this gap by directly optimizing the prior-assignment laws, yielding the state-of-the-art average score of $38.35\%$. This performance not only exceeds the human baseline but also surpasses DoReMi ($38.18\%$), validating the robustness of our framework in challenging scenarios where standard scaling assumptions are compromised. Furthermore, ByDoRe maintains peak computational efficiency, requiring only $1.44 \times 10^{16}$ FLOPs, the lowest overhead among all automated baselines, while delivering the most accurate results. By internalizing stable importance patterns through the Gamma-Dirichlet hierarchy, ByDoRe provides a highly cost-effective and superior solution for specialized domain scheduling tasks.

\subsection{Mechanisms and Ablation Analysis}
\label{sec:mechanisms}

\begin{figure*}[t]
	\centering
	\begin{subfigure}[b]{0.47\textwidth}
		\centering
		\includegraphics[width=\textwidth]{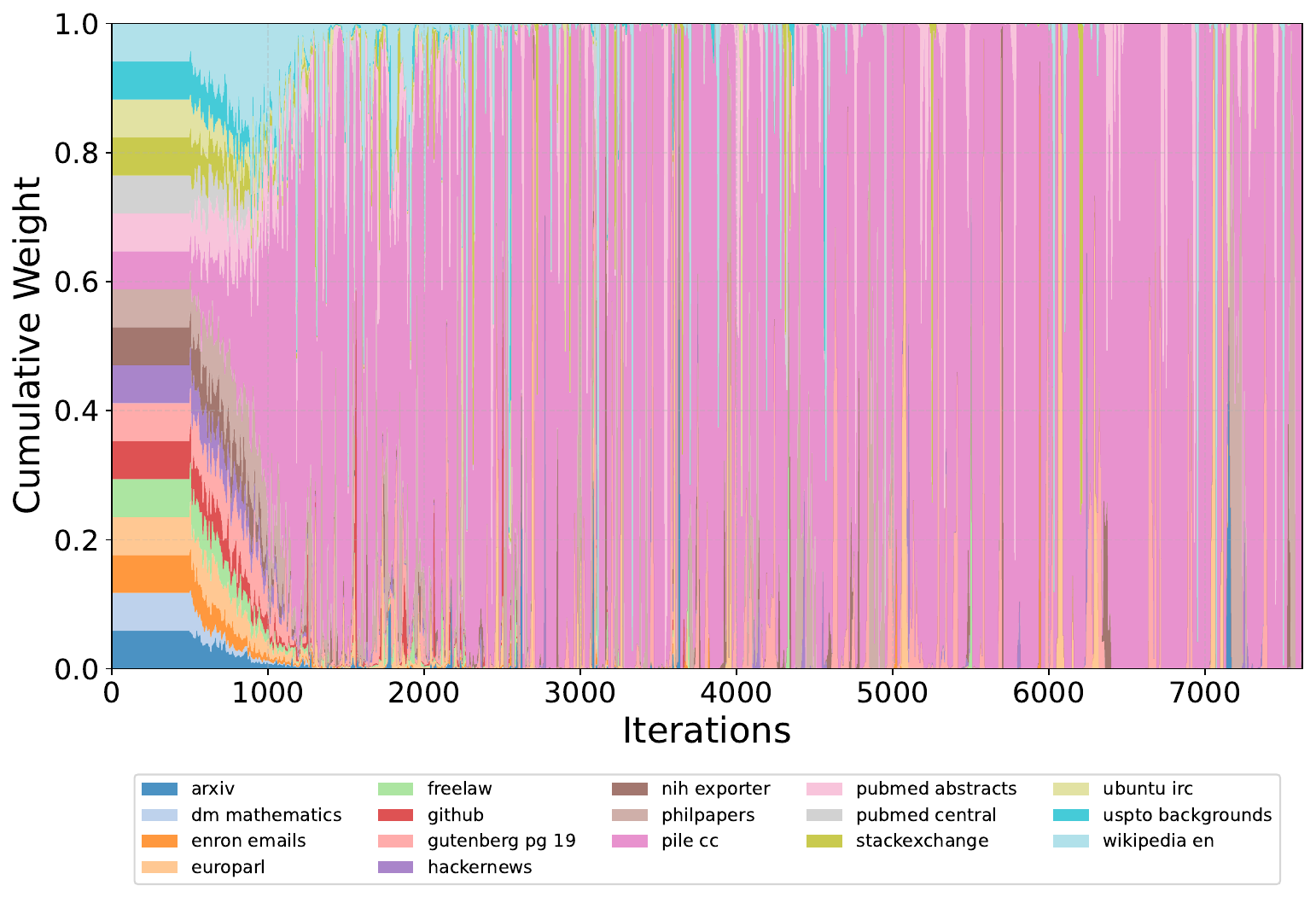} 
		\caption{Weights evolution without probabilistic modeling} 
		\label{fig:weight_sensitivity}
	\end{subfigure}
	\hfill 
	\begin{subfigure}[b]{0.515\textwidth}		
		\centering
		\includegraphics[width=\textwidth]{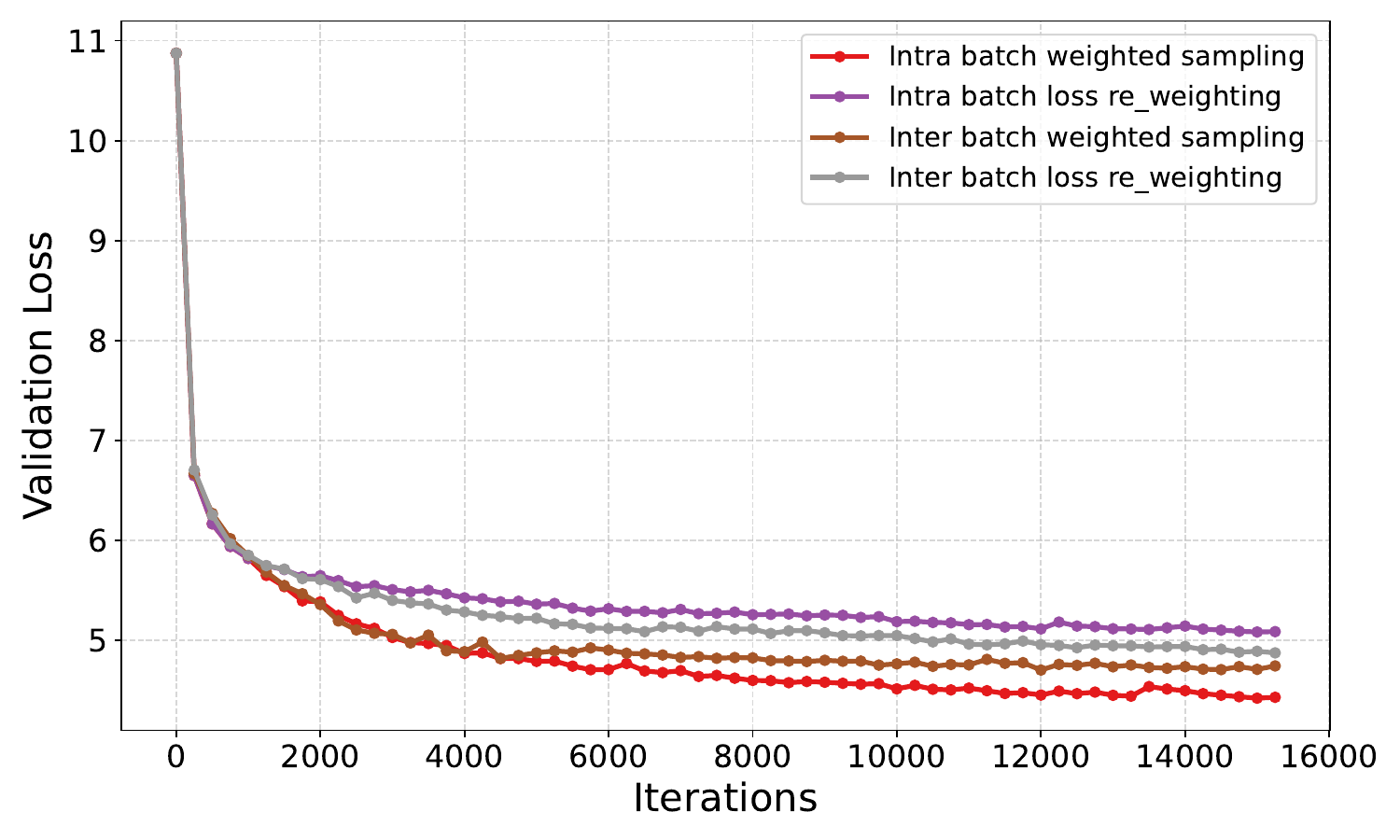} 
		\caption{Validation loss of different scheduling strategies} 
		\label{fig:strategies_godam}
	\end{subfigure}	
	\caption{Ablation studies on probabilistic modeling and scheduling efficiency: (a) Learning trajectories of a non-probabilistic baseline, showing high-variance jitters and overfitting to the \textit{Pile-CC} domain; (b) Convergence comparison of four scheduling strategies, where our intra-batch weighted sampling achieves the fastest and lowest validation loss.}
	\label{fig:ablation_studies} 
\end{figure*}

To provide a comprehensive understanding of ByDoRe, we conduct ablation studies focusing on the stabilizing role of probabilistic modeling and the necessity of intervening in the physical data sampling process.

\textbf{Suppressing Optimization Jitter via Probabilistic Modeling.} 
To verify the stabilizing effect of hierarchical Bayesian domain weighting, we compare ByDoRe with a baseline that removes the probabilistic layer and directly optimizes discrete weights via point-wise estimation. As illustrated in the weights learning trajectories in Fig.~\ref{fig:weight_sensitivity}, the variant without probabilistic modeling exhibits violent oscillations and jitters, reflecting an over-sensitivity to batch-level variance. This leads to a failure in regularization, where the model overfits the validation distribution by assigning it a dominant, near-unity weight. In contrast, ByDoRe framework utilizes the Bayesian formulation to smooth the optimization landscape, maintaining a generalizable domain weights distribution even under high data variance as in Fig~\ref{fig:motivation}(b).

\textbf{Physical Sampling Intervention and Convergence Speed.} 
We investigate how learned weights should be applied to the model's training by evaluating four distinct scheduling strategies: 
(i) Intra-batch weighted sampling, where domain proportions are strictly maintained within each training batch; 
(ii) Inter-batch weighted sampling, where each batch consists of a single domain sampled according to $\mathbf{w}$; 
(iii) Intra-batch loss re-weighting, which weights domain losses but uses uniform sampling; 
and (iv) Inter-batch loss re-weighting, which weights domain losses while employing uniform sampling for batches consisting of a single domain. 
See the appendix~\ref{app:strategies_setting} for detailed experimental setup.
As shown in Fig.~\ref{fig:strategies_godam}, intra-batch weighted sampling is the most efficient strategy, achieving significantly lower validation loss than loss-re-weighting alternatives. This confirms that intervening in the physical data reading process provides the model with higher-quality gradients. Crucially, without this real-time intervention, the optimization suffers from a drastic reduction in convergence speed, requiring more iterations to identify an effective schedule. By feeding updated weights back to intervene in sampling, ByDoRe creates a dynamic feedback loop that allows the prior prediction network to observe the model's response in real-time, ensuring rapid and stable convergence.

\section{Related Work}
Due to space limitations, we briefly discuss related work of domain weighting in machine learning here and leave other related work analysis in Appendix~\ref{sec:related_work}.

The optimization of data composition through domain weighting has evolved from sample re-weighting in traditional supervised learning to the complex scheduling challenges of LLMs era.
Weighting training samples is a long-standing problem in machine learning, primarily addressed through meta-learning techniques \cite{ren2018learning, shu2019metaweightnet}. These frameworks establish a rigorous data-driven foundation by learning an importance-weighting policy that minimizes validation loss to handle issues like label noise or class imbalance \cite{wang2020tanda, koh2017understanding}. While theoretically robust in manageable data regimes, the immense scale of LLMs pre-training datasets rendered these methods computationally prohibitive. Consequently, early LLMs such as GPT-3 \cite{brown2020language} and Llama 2 \cite{touvron2023llama2} initially reverted to manual heuristics, where data mixtures were determined by human intuition based on perceived dataset quality.

To move beyond manual tuning, recent research has focused on exploring automated weighting strategies on smaller-scale proxy models through two primary paths~\cite{albalak2024survey, liu2025rethinking, chen2024datajuicer, xie2023doremi, liu2024regmix, kang2024autoscale, ye2024datamixing, ge2024bimix, belenki2025mde, hoffmann2022chinchilla, gu2025data, shukor2025scaling}. The first path is direct optimization, exemplified by DoReMi \cite{xie2023doremi} and DoGE \cite{fan2024doge}, which solve for domain weights directly on the target scale. However, these methods are notoriously unstable due to extreme batch-level variance and incur massive computational overhead, often requiring exhaustive data traversal to achieve convergence. The second path is function-fitting methods, which estimate loss as a function of domain weights and interpolate performance
to determine optimal data compositions. They introduces structural constraints such as rank invariance \cite{liu2024regmix, belenki2025mde} or scaling laws \cite{kang2024autoscale} to generalize large model and data scales. Although efficient, these structural assumptions are frequently violated in practice, leading to significant estimation bias and suboptimal mixtures. ByDoRe bridges this gap by shifting the optimization paradigm from deterministic modeling to a probabilistic modeling framework, achieving superior stability and lower computational cost by effectively suppressing variance-induced jitters.

\section{Conclusion}
\label{sec:conclusion}

In this work, we addressed the critical challenge of data mixtures for LLMs pre-training by introducing a Bayesian domain weighting framework. By formulating domain weights through a Gamma-Dirichlet hierarchical Bayesian model and adaptively inferring posterior distribution from observations, we effectively suppress optimization jitters and marginalize out batch-level stochasticity. Experimental evaluations across diverse data regimes demonstrate that proposed method achieves stable and efficient domain weights configuration with significantly lower computational overhead than existing methods. Furthermore, we demonstrate that this method offers valuable insights and mechanisms for weighting learning. Our study suggest a promising direction for data optimization in the LLMs era, e.g.,  data mixing and selection across all stages of LLM training.

\section*{Impact Statement}
This paper presents work whose goal is to advance the field of Machine
Learning. There are many potential societal consequences of our work, none
which we feel must be specifically highlighted here.

%
%

\bibliographystyle{icml2026}
\bibliography{ref}

\newpage
\appendix
\onecolumn

\section*{Overview of Appendices}

This appendix provides supplementary mathematical derivations, proofs, and experimental details to support the main text. 
\begin{itemize}
	\item \textbf{Appendix \ref{app:first_order}} derives the first-order meta-gradient approximation used to bypass the Hessian bottleneck.
	\item \textbf{Appendix \ref{app:DPT1}} provides rigorous convergence proofs for both the meta-objective and the training objective.
	\item \textbf{Appendix \ref{sec:related_work}} contextualizes ByDoRe within the landscape of domain weighting and probabilistic modeling.
	\item \textbf{Appendix \ref{app:dataset_details}} and \textbf{\ref{app:validation_datasets}} detail the training datasets and downstream evaluation benchmarks, respectively.
	\item \textbf{Appendix \ref{app:strategies_setting}} describes the implementation of various data scheduling strategies.
	\item \textbf{Appendix \ref{app:violation}} provides an empirical analysis of scaling law violations that necessitate our probabilistic approach.
\end{itemize}

\section{Detailed Derivation of the First-Order Meta-Gradient Approximation}
\label{app:first_order}

This appendix provides a rigorous mathematical derivation of the first-order approximation for updating the prior-setting laws $\Phi$. Unlike traditional deterministic weighting, ByDoRe optimizes a variational posterior $q_\Phi(\mathbf{w})$ via the Evidence Lower Bound (ELBO). We demonstrate how the resulting hypergradient, which involves a complex Hessian-vector product, is efficiently approximated using finite differences of the training loss.

\subsection{Variational Meta-Objective}

Based on the Bayesian framework defined in Section~\ref{sec:bydore}, the prior-setting model $\Phi$ is optimized to maximize the ELBO $\mathcal{L}(q_\Phi)$:
\begin{equation}
	\mathcal{L}(q_\Phi) \approx \mathbb{E}_{q_\Phi(\mathbf{w})} \left[ \log p(\mathcal{D}_{\text{val}} \mid \hat{\theta}(\mathbf{w})) \right] - \text{KL}(q_\Phi(\mathbf{w}) \| p(\mathbf{w})),
\end{equation}
where $\hat{\theta}(\mathbf{w})$ is the proxy model parameters resulting from training with mixture $\mathbf{w}$. To make the optimization tractable, we utilize the analytical expectation $\bar{\mathbf{w}}(\Phi) = \mathbb{E}_{q_\Phi}[\mathbf{w}]$ as a differentiable surrogate for the distribution. The meta-objective is simplified as:
\begin{equation}
	\max_{\Phi} \quad \mathcal{J}(\Phi) = \mathcal{L}_{\text{val}}(\hat{\theta}(\bar{\mathbf{w}}(\Phi))) - \lambda \mathcal{R}_{\text{KL}}(\Phi),
	\label{eq:variational_obj}
\end{equation}
where $\mathcal{L}_{\text{val}} = \log p(\mathcal{D}_{\text{val}} \mid \cdot)$ and $\mathcal{R}_{\text{KL}}$ is the KL-regularization term between the Gamma-Dirichlet posterior and the uniform prior.

\subsection{Virtual Update and the Hessian Bottleneck}

Following the bi-level optimization paradigm, we approximate the optimal $\hat{\theta}$ using a single-step virtual update from the current state $\theta_t$:
\begin{equation}
	\tilde{\theta}_{t+1}(\Phi) = \theta_t - \eta \nabla_{\theta} \mathcal{L}_{\text{tr}}(\theta_t, \bar{\mathbf{w}}(\Phi)).
\end{equation}
To update $\Phi$ via gradient ascent $\Phi_{t+1} = \Phi_t + \gamma \nabla_\Phi \mathcal{J}(\Phi)$, we compute the gradient of Eq.~\ref{eq:variational_obj}:
\begin{equation}
	\nabla_\Phi \mathcal{J}(\Phi) = \underbrace{\frac{\partial \mathcal{L}_{\text{val}}(\tilde{\theta}_{t+1})}{\partial \tilde{\theta}_{t+1}(\Phi)} \frac{\partial \tilde{\theta}_{t+1}(\Phi)}{\partial \Phi}}_{\text{Hypergradient term}} - \lambda \nabla_\Phi \mathcal{R}_{\text{KL}}(\Phi).
\end{equation}
Differentiating the virtual update with respect to $\Phi$ yields:
\begin{equation}
	\frac{\partial \tilde{\theta}_{t+1}(\Phi)}{\partial \Phi} = - \eta \frac{\partial^2 \mathcal{L}_{\text{tr}}(\theta_t, \bar{\mathbf{w}}(\Phi))}{\partial \theta \partial \Phi}.
\end{equation}
Substituting this back, and letting $\Delta = \nabla_\theta \mathcal{L}_{\text{val}}(\tilde{\theta}_{t+1})$ denote the validation gradient, the hypergradient term becomes:
\begin{equation}
	\nabla_\Phi \mathcal{L}_{\text{val}} = - \eta \Delta^\top \frac{\partial^2 \mathcal{L}_{\text{tr}}(\theta_t, \bar{\mathbf{w}}(\Phi))}{\partial \theta \partial \Phi}.
	\label{eq:bayesian_hessian_bottleneck}
\end{equation}
The term $\frac{\partial^2 \mathcal{L}_{\text{tr}}}{\partial \theta \partial \Phi}$ is a large-scale Hessian matrix whose direct computation is prohibitive.

\subsection{Finite Difference Approximation}

To resolve the bottleneck in Eq.~\ref{eq:bayesian_hessian_bottleneck}, we define a gradient function $G(\theta) = \nabla_\Phi \mathcal{L}_{\text{tr}}(\theta, \bar{\mathbf{w}}(\Phi))$. We apply a central difference approximation to the second-order mixed derivative. For a small scalar $\epsilon > 0$, the Taylor expansion of $G(\theta)$ around $\theta_t$ in the direction of the validation gradient $\Delta$ gives:
\begin{align}
	G(\theta_t + \epsilon \Delta) &\approx G(\theta_t) + \epsilon \frac{\partial^2 \mathcal{L}_{\text{tr}}}{\partial \Phi \partial \theta} \Delta, \\
	G(\theta_t - \epsilon \Delta) &\approx G(\theta_t) - \epsilon \frac{\partial^2 \mathcal{L}_{\text{tr}}}{\partial \Phi \partial \theta} \Delta.
\end{align}
Subtracting these two equations and rearranging yields:
\begin{equation}
	\frac{\partial^2 \mathcal{L}_{\text{tr}}}{\partial \Phi \partial \theta} \Delta \approx \frac{\nabla_\Phi \mathcal{L}_{\text{tr}}(\theta_t + \epsilon \Delta, \bar{\mathbf{w}}) - \nabla_\Phi \mathcal{L}_{\text{tr}}(\theta_t - \epsilon \Delta, \bar{\mathbf{w}})}{2 \epsilon}.
\end{equation}
Substituting this into the update rule, the final iterative formula for the prior network $\Phi$ is:
\begin{equation}
	\Phi_{t+1} = \Phi_t - \gamma \left( \frac{\eta}{2\epsilon} \left[ \nabla_\Phi \mathcal{L}_{\text{tr}}(\theta_t + \epsilon \Delta, \bar{\mathbf{w}}) - \nabla_\Phi \mathcal{L}_{\text{tr}}(\theta_t - \epsilon \Delta, \bar{\mathbf{w}}) \right] + \lambda \nabla_\Phi \mathcal{R}_{\text{KL}}(\Phi) \right).
\end{equation}
This formulation allows ByDoRe to optimize the Bayesian prior using only first-order gradients of the training loss, effectively capturing the domain structural importance while maintaining computational efficiency.

\section{Detailed Proof of Convergence Theorem}
\label{app:DPT1}

\subsection{Preliminaries and Assumptions}

\noindent
\textbf{Data and Domain Decomposition.}
We consider a training corpus $\mathcal D=\bigcup_{k=1}^K \mathcal D_k$
partitioned into $K$ domains, where $\mathcal D_k$ denotes the data from the
$k$-th domain. For a model parameterized by $\theta$, we define the domain-wise
(expected) loss as
\[
\ell_k(\theta)
:=
\mathbb E_{x\sim \mathcal D_k}\big[\ell_\theta(x)\big],
\]
where $\ell_\theta(x)$ is the per-sequence auto-regressive (next-token) loss used
in the main text. Accordingly, the expected training objective under a domain
weight vector $\mathbf w=(w_1,\dots,w_K)\in\Delta^{K-1}$ can be written as
\[
\mathcal L_{\mathrm{tr}}(\theta,\mathbf w)
=
\sum_{k=1}^K w_k\,\ell_k(\theta),
\]
which is consistent with the definition
$\mathbb E_{x\sim p_{\mathbf w}(x)}[\ell_\theta(x)]\triangleq
\mathcal L_{\mathrm{tr}}(\theta,\mathbf w)$ in the main text.

\textbf{Training Objective under $q_\Phi(\mathbf w)$.}
We therefore analyze convergence to stationary points of the expected training objective denoted as

\[
{\mathcal{L}}_{\mathrm{tr}}(\theta, \Phi)
:=\mathcal{L}_{\text{tr}}(\theta,\bar{\mathbf w}(\Phi))=
\sum_{k=1}^K \bar{w}_k(\Phi)\,\ell_k(\theta),
\]
where
$\bar{\mathbf w}(\Phi)=\mathbb E_{\mathbf w\sim q_\Phi(\mathbf w)}[\mathbf w]$
denotes the expected domain weights induced by the prior-setting laws $\Phi$
and is given in closed form by Eq.~(\ref{eq:expectation_w}) of the main text.

\paragraph{Meta Objective}
Throughout this appendix, we analyze a single meta objective defined as the
\emph{negative ELBO} associated with the variational posterior $q_\Phi(\mathbf w)$.
Specifically, we adopt the standard convention that each validation loss is the
negative log-likelihood,
$
L_i^{\mathrm{nval}}(\theta) := - \log p(x_i^{\mathrm{val}} \mid \theta),
$
and define
\begin{equation}
	\label{eq:Lmeta_def}
	\mathcal{L}_{\mathrm{meta}}(\Phi)
	:=
	\mathcal{L}_{\mathrm{nval}}(\Phi)
	+
	\lambda \,\mathrm{KL}\!\left(q_\Phi(\mathbf w)\,\|\,p(\mathbf w)\right),
\end{equation}
where
\[
\mathcal{L}_{\mathrm{nval}}(\Phi)
=
\frac{1}{M} \sum_{i=1}^M L_i^{\mathrm{nval}}(\hat{\theta}(\Phi)),
\qquad
\hat{\theta}(\Phi)
=
\theta - \eta_t \nabla_{\theta}\mathcal{L}_{\mathrm{tr}}(\theta,\mathbf w(\Phi)).
\]
By construction, $\mathcal{L}_{\mathrm{meta}}(\Phi) = - \mathcal L(\Phi)$, where
$\mathcal L(\Phi)$ denotes the ELBO maximized in the main text.
Therefore, minimizing $\mathcal{L}_{\mathrm{meta}}(\Phi)$ is equivalent to
maximizing the ELBO, and both objectives admit identical stationary points.
All convergence guarantees derived below for
$\nabla_\Phi \mathcal{L}_{\mathrm{meta}}(\Phi_t)$
thus directly apply to the ELBO objective $\mathcal L(\Phi_t)$.

\paragraph{Sign Convention.}
Throughout this appendix, we analyze the minimization of the meta objective
$\mathcal L_{\mathrm{meta}}(\Phi)$ defined in
Eq.~\eqref{eq:Lmeta_def}, which is the negative ELBO.
Accordingly, all parameter updates in this section are written in the standard
gradient-descent form.
This explains the subtraction in the update
$\Phi_{t+1}=\Phi_t-\gamma_t g_{\Phi_t}$.
In contrast, the main text formulates the update as a gradient-ascent step
on the ELBO; the two forms are equivalent since
$\mathcal L_{\mathrm{meta}}(\Phi)=-\mathcal L(\Phi)$.

\paragraph{Updates.}
We analyze the alternating updates in Algorithm~\ref{alg:ByDoRe}:
\[
\theta_{t+1}
=
\theta_t
-
\eta_t g_{\theta_t},
\qquad
\Phi_{t+1}
=
\Phi_t
-
\gamma_t g_{\Phi_t},
\]
where
\[
g_{\theta_t}
:=
\nabla_\theta \mathcal L_{\mathrm{tr}}(\theta_t; B_t),
\qquad
B_t\sim p_{\mathbf w_t},
\]
and the meta-gradient estimator is
\begin{equation}
	\label{eq:gphi_def_aligned}
	g_{\Phi_t}
	:=
	\frac{\eta_t}{2\epsilon}
	\Big[
	\nabla_{\Phi}
	\mathcal L_{\mathrm{tr}}(\theta_t + \epsilon \Delta,\bar{\mathbf w}_t)
	-
	\nabla_{\Phi}
	\mathcal L_{\mathrm{tr}}(\theta_t - \epsilon \Delta,\bar{\mathbf w}_t)
	\Big]
	+
	\lambda
	\nabla_{\Phi}
	\mathcal R_{\mathrm{KL}}(\Phi_t),
\end{equation}
with
\[
\Delta := \nabla_{\theta}\mathcal L_{\mathrm{nval}}(\theta_t),
\]
denoting the gradient of the validation loss with respect to the model parameters,
and $\epsilon>0$ the finite-difference step size.

\paragraph{Assumptions.}
We assume the following standard regularity conditions:

\begin{enumerate}
	\item \textbf{Lipschitz Smoothness.}
	The expected training objective $\mathcal L_{\mathrm{tr}}(\theta,\Phi)$ and each
	validation loss $L_i^{\mathrm{nval}}(\theta)$ are twice continuously differentiable
	and $L$-smooth with respect to $\theta$. That is, for all $\theta_1,\theta_2$ and
	any $\Phi$,
	\[
	\|\nabla_{\theta}\mathcal{L}_{\mathrm{tr}}(\theta_1,\Phi)
	-
	\nabla_{\theta}\mathcal{L}_{\mathrm{tr}}(\theta_2,\Phi)\|
	\le
	L\|\theta_1-\theta_2\|,
	\]
	and
	\[
	\|\nabla_{\theta} L_i^{\mathrm{nval}}(\theta_1)
	-
	\nabla_{\theta} L_i^{\mathrm{nval}}(\theta_2)\|
	\le
	L\|\theta_1-\theta_2\|,
	\qquad \forall i\in\{1,\dots,M\}.
	\]
	Equivalently,
	\[
	\|\nabla_{\theta}^2\mathcal{L}_{\mathrm{tr}}(\theta,\Phi)\| \le L,
	\qquad
	\|\nabla_{\theta}^2 L_i^{\mathrm{nval}}(\theta)\| \le L.
	\]
	
	\item \textbf{Bounded Prior Mapping.}
	The mapping
	$\bar{\mathbf{w}}(\Phi)=\mathbb{E}_{\mathbf w\sim q_\Phi}[\mathbf{w}]$
	is differentiable with a $\delta$-bounded gradient and twice continuously
	differentiable with its Hessian operator norm bounded by $B$.
	
	\item \textbf{Bounded Gradients.}
	The gradients of the loss functions with respect to the training and meta
	parameters are uniformly bounded. In particular, there exists a constant
	$\rho>0$ such that
	\[
	\|\nabla_{\theta}\mathcal{L}_{\mathrm{tr}}(\theta,\Phi)\|_2 \le \rho,
	\qquad
	\|\nabla_{\Phi}\mathcal{L}_{\mathrm{meta}}(\Phi)\|_2 \le \rho,
	\]
	for all admissible $(\theta,\Phi)$.
	
	\item \textbf{KL Regularizer Smoothness.}
	The KL regularization term $\mathcal{R}_{\mathrm{KL}}(\Phi)$ is twice continuously
	differentiable and $L_{\mathrm{KL}}$-smooth:
	\[
	\|\nabla_\Phi \mathcal{R}_{\mathrm{KL}}(\Phi_1)-\nabla_\Phi \mathcal{R}_{\mathrm{KL}}(\Phi_2)\|
	\le L_{\mathrm{KL}}\|\Phi_1-\Phi_2\|,\qquad \forall \Phi_1,\Phi_2,
	\]
	equivalently, $\|\nabla_\Phi^2 \mathcal{R}_{\mathrm{KL}}(\Phi)\|\le L_{\mathrm{KL}}$.
	
	\item \textbf{Bounded Loss.}
	Each domain-wise component loss is uniformly bounded:
	$0 \le \ell_k(\theta) \le G_\ell$ for all $k\in\{1,\dots,K\}$ and all $\theta$.
	Consequently,
	\[
	\sum_{k=1}^K \ell_k(\theta_{t+1}) \le K G_\ell .
	\]
	
	\item \textbf{Estimators.}
	The stochastic gradient $g_{\theta_t}$ is an unbiased estimator of
	$\nabla_{\theta}\mathcal{L}_{\mathrm{tr}}(\theta_t,\Phi_t)$:
	\begin{equation}
		g_{\theta_t}
		=
		\nabla_{\theta}\mathcal{L}_{\mathrm{tr}}(\theta_t,\Phi_t)
		+
		\xi^{(t)},\qquad
		\mathbb{E}[\xi^{(t)}]=0,\quad
		\mathbb{E}\|\xi^{(t)}\|^2 \le \sigma_{\theta}^2.
	\end{equation}
	The meta-gradient estimator $g_{\Phi_t}$ is a biased stochastic estimator of
	$\nabla_{\Phi}\mathcal{L}_{\mathrm{meta}}(\Phi_t)$:
	\begin{equation}
		g_{\Phi_t}
		=
		\nabla_{\Phi}\mathcal{L}_{\mathrm{meta}}(\Phi_t)
		+
		\psi^{(t)} + m_t,
		\qquad
		\mathbb{E}[\psi^{(t)}]=0,\quad
		\mathbb{E}\|\psi^{(t)}\|^2 \le \sigma_{\Phi}^2,\quad
		\mathbb{E}\|m_t\|^2 \le c_b\,\epsilon^4.
	\end{equation}
\end{enumerate}

\subsection{Lemma 1: Lipschitz Continuity of the Meta-gradient}
Before establishing convergence, we prove that the gradient of the meta-loss with respect to the prior laws $\Phi$ is Lipschitz continuous.

\begin{lemma}
	Under Assumptions 1--3, the meta-gradient $\nabla_{\Phi} \mathcal{L}_{\text{meta}}(\Phi)$ is Lipschitz continuous.
	That is, there exists a constant $L_{\Phi} > 0$ such that for any $\Phi_1, \Phi_2$,
	\begin{equation}
		\big\| \nabla_{\Phi} \mathcal{L}_{\text{meta}}(\Phi_1) - \nabla_{\Phi} \mathcal{L}_{\text{meta}}(\Phi_2) \big\|
		\le
		L_{\Phi} \, \|\Phi_1 - \Phi_2\|.
	\end{equation}
\end{lemma}

\begin{proof}
	Recall that the meta-loss is defined as
	\begin{equation}
		\mathcal{L}_{\text{meta}}(\Phi)
		=
		\mathcal{L}_{\text{nval}}(\Phi)
		+
		\lambda \mathcal{R}_{\mathrm{KL}}(\Phi),
		\qquad
		\mathcal{L}_{\text{nval}}(\Phi)
		= \frac{1}{M} \sum_{i=1}^M L_i^{\text{nval}}(\hat{\theta}(\Phi)),
	\end{equation}
	where
	\begin{equation}
		\hat{\theta}(\Phi)
		= \theta - \eta_t \nabla_{\theta} \mathcal{L}_{\text{tr}}(\theta, \bar{\mathbf w}(\Phi)),
		\qquad
		\mathcal{L}_{\text{tr}}(\theta, \bar{\mathbf w}(\Phi))
		= \sum_{k=1}^K \bar w_k(\Phi)\, \ell_k(\theta),
	\end{equation}
	and the expected weight mapping is
	\[
	\bar{\mathbf w}(\Phi)=\mathbb{E}_{\mathbf w\sim q_\Phi(\mathbf w)}[\mathbf w].
	\]
	
	We first bound the smoothness of $\nabla_\Phi \mathcal{L}_{\text{nval}}(\Phi)$ as in the original proof.
	By the chain rule, the gradient of a single validation loss with respect to $\Phi$ is
	\begin{equation}
		\nabla_{\Phi} L_i^{\text{nval}}(\hat{\theta}(\Phi))
		=
		-\eta_t
		\sum_{k=1}^K
		\Big(
		\nabla_{\hat{\theta}} L_i^{\text{nval}}(\hat{\theta})^\top
		\nabla_{\theta} \ell_k(\theta)
		\Big)
		\nabla_{\Phi} \bar w_k(\Phi).
	\end{equation}
	Define
	\begin{equation}
		G_{ik}(\theta)
		:= \nabla_{\hat{\theta}} L_i^{\text{nval}}(\hat{\theta})^\top
		\nabla_{\theta} \ell_k(\theta),
	\end{equation}
	then
	\begin{equation}
		\nabla_{\Phi} L_i^{\text{nval}}(\hat{\theta}(\Phi))
		=
		-\eta_t \sum_{k=1}^K G_{ik}(\theta)\, \nabla_{\Phi} \bar w_k(\Phi).
	\end{equation}
	
	Taking the derivative with respect to $\Phi$ again yields
	\begin{equation}
		\nabla_{\Phi}^2 L_i^{\text{nval}}
		=
		-\eta_t
		\sum_{k=1}^K
		\Big[
		\nabla_{\Phi} G_{ik}(\theta)\, \nabla_{\Phi} \bar w_k(\Phi)
		+
		G_{ik}(\theta)\, \nabla_{\Phi}^2 \bar w_k(\Phi)
		\Big].
	\end{equation}
	
	We bound the two terms separately.
	By Assumption~1, $L_i^{\mathrm{nval}}$ is $L$-smooth with respect to $\hat\theta$,
	and by Assumption~3, the gradients of $\ell_k$ are uniformly bounded by $\rho$.
	Moreover, by Assumption~2, the prior mapping satisfies
	$\|\nabla_{\Phi}\bar w_k(\Phi)\|\le \delta$.
	
	Note that $G_{ik}$ depends on $\Phi$ only through $\hat\theta(\Phi)$.
	Applying the chain rule yields
	\[
	\nabla_{\Phi} G_{ik}
	=
	\nabla_{\hat\theta}^2 L_i^{\mathrm{nval}}(\hat\theta)\,
	\nabla_{\Phi}\hat\theta\,
	\nabla_\theta \ell_k(\theta),
	\]
	where
	\[
	\nabla_{\Phi}\hat\theta
	=
	-\eta_t \sum_{j=1}^K
	\nabla_\theta \ell_j(\theta)\,\nabla_{\Phi}\bar w_j(\Phi).
	\]
	Therefore,
	\[
	\|\nabla_{\Phi}\hat\theta\|
	\le
	\eta_t \sum_{j=1}^K
	\|\nabla_\theta \ell_j(\theta)\|\,
	\|\nabla_{\Phi}\bar w_j(\Phi)\|
	\le
	\eta_t K \rho \delta.
	\]
	Combining the above bounds yields
	\[
	\big\|
	\nabla_{\Phi} G_{ik}(\theta)\,\nabla_{\Phi}\bar w_k(\Phi)
	\big\|
	\le
	\eta_t L \rho^2 K \delta^2.
	\]
	For the second term, using the $B$-bounded Hessian of the prior mapping in
	Assumption~2 and $|G_{ik}(\theta)| \le \rho^2$, we have
	\[
	\big\|
	G_{ik}(\theta)\, \nabla_{\Phi}^2 \bar w_k(\Phi)
	\big\|
	\le
	\rho^2 B.
	\]
	Assuming a constant step size $\eta$ (i.e., $\eta_t \equiv \eta$), summing over $k$ gives
	\[
	\|\nabla_{\Phi}^2 L_i^{\mathrm{nval}}\|
	\le
	\eta K \rho^2 (\eta L K \delta^2 + B).
	\]
	Hence $\nabla_\Phi \mathcal{L}_{\mathrm{nval}}(\Phi)$ is Lipschitz continuous with
	constant
	\[
	L_{\Phi}^{\mathrm{nval}}
	=
	\eta K \rho^2 (\eta L K\delta^2 + B).
	\]
	
	Finally, since
	\[
	\nabla_\Phi \mathcal{L}_{\text{meta}}(\Phi)
	=
	\nabla_\Phi \mathcal{L}_{\text{nval}}(\Phi)
	+
	\lambda \nabla_\Phi \mathcal{R}_{\mathrm{KL}}(\Phi),
	\]
	if $\mathcal{R}_{\mathrm{KL}}(\Phi)$ is twice continuously differentiable with
	$\|\nabla_\Phi^2 \mathcal{R}_{\mathrm{KL}}(\Phi)\|\le L_{\mathrm{KL}}$,
	then $\nabla_\Phi \mathcal{L}_{\text{meta}}(\Phi)$ is Lipschitz with
	\[
	L_\Phi
	=
	L_{\Phi}^{\mathrm{nval}}
	+
	\lambda L_{\mathrm{KL}}.
	\]
	This completes the proof.
\end{proof}

\subsection{Convergence Analysis of Meta-Objective}
\begin{theorem}
	Suppose Assumptions~1--6 hold.
	Let the learning rates $\{\gamma_t\}_{t=1}^T$ and $\{\eta_t\}_{t=1}^T$ be monotonically
	non-increasing sequences defined as
	\[
	\gamma_t=\min\Big\{\frac{1}{6L_\Phi},\,\frac{c_1}{\sqrt{T}}\Big\},\qquad
	\eta_t=\min\Big\{\frac{1}{L},\,\frac{c_2}{\sqrt{T}}\Big\},
	\]
	for some constants $c_1,c_2>0$ such that
	\[
	\frac{\sqrt{T}}{c_1}\ge L_\Phi,
	\qquad
	\frac{\sqrt{T}}{c_2}\ge L.
	\]
	Then the proposed algorithm guarantees
	\begin{equation}
		\min_{0\le t\le T-1}
		\mathbb{E}\Big[\big\|\nabla_{\Phi}\mathcal{L}_{\text{meta}}(\Phi_t)\big\|_2^2\Big]
		\;\le\;
		\mathcal{O}\!\Big(\frac{C}{\sqrt{T}}\Big)
		\;+\;
		\mathcal{O}(\epsilon^4),
		\label{eq:meta_stationarity}
	\end{equation}
	where $C$ is a positive constant independent of $T$ and the convergence process.
	Here, $\epsilon$ denotes the finite-difference step size used for gradient
	approximation.
	\emph{Equivalently, since $\mathcal{L}_{\text{meta}}(\Phi) = -\mathcal L(\Phi)$,
		the same bound holds for the ELBO objective $\mathcal L(\Phi)$ used in the main text:}
	\begin{equation}
		\min_{0\le t\le T-1}
		\mathbb{E}\Big[\big\|\nabla_{\Phi}\mathcal{L}(\Phi_t)\big\|_2^2\Big]
		\;\le\;
		\mathcal{O}\!\Big(\frac{C}{\sqrt{T}}\Big)
		\;+\;
		\mathcal{O}(\epsilon^4),
		\label{eq:meta_stationarity}
	\end{equation}
	
\end{theorem}

\begin{proof}
	Consider the one-step change of the meta objective along consecutive iterates:
	\begin{equation}
		\begin{aligned}
			\mathcal{L}_{\text{meta}}\big(\theta_{t+1}(\Phi_{t+1})\big)-\mathcal{L}_{\text{meta}}\big(\theta_t(\Phi_t)\big)
			=&\ 
			\Big[\mathcal{L}_{\text{meta}}\big(\theta_{t+1}(\Phi_{t+1})\big)-\mathcal{L}_{\text{meta}}\big(\theta_{t+1}(\Phi_t)\big)\Big]\\
			&\ +\Big[\mathcal{L}_{\text{meta}}\big(\theta_{t+1}(\Phi_t)\big)-\mathcal{L}_{\text{meta}}\big(\theta_t(\Phi_t)\big)\Big].
		\end{aligned}
		\label{eq:decomp_no_I_II}
	\end{equation}
	First, we bound the first term in \eqref{eq:decomp_no_I_II}, which corresponds to
	the change in $\mathcal{L}_{\text{meta}}$ induced by the update of $\Phi$.
	
	By Lemma~1, $\mathcal{L}_{\text{meta}}(\Phi)$ is $L_{\Phi}$-smooth with respect to $\Phi$.
	Using the update rule $\Phi_{t+1}=\Phi_t-\gamma_t g_{\Phi_t}$ and the smoothness inequality,
	we obtain
	\begin{equation}
		\mathcal{L}_{\text{meta}}(\Phi_{t+1})-\mathcal{L}_{\text{meta}}(\Phi_t)
		\le
		-\gamma_t\big\langle \nabla_\Phi\mathcal{L}_{\text{meta}}(\Phi_t), g_{\Phi_t}\big\rangle
		+\frac{L_\Phi}{2}\gamma_t^2\|g_{\Phi_t}\|^2.
		\label{eq:phi_smooth}
	\end{equation}
	
	Recall that the meta-gradient estimator admits the decomposition
	\[
	g_{\Phi_t}
	=
	\nabla_\Phi\mathcal{L}_{\text{meta}}(\Phi_t)
	+\psi^{(t)}+m_t,
	\]
	where $\mathbb{E}_t[\psi^{(t)}]=0$ and $m_t$ denotes the finite-difference bias term.
	By $L_\Phi$-smoothness of $\mathcal{L}_{\text{meta}}(\cdot)$, for the update
	$\Phi_{t+1}=\Phi_t-\gamma_t g_{\Phi_t}$ we have the descent lemma
	\begin{equation}
		\mathcal{L}_{\text{meta}}(\Phi_{t+1})
		\le
		\mathcal{L}_{\text{meta}}(\Phi_t)
		-\gamma_t\langle \nabla_\Phi \mathcal{L}_{\text{meta}}(\Phi_t), g_{\Phi_t}\rangle
		+\frac{L_\Phi}{2}\gamma_t^2\|g_{\Phi_t}\|^2 .
		\label{eq:descent_phi}
	\end{equation}
	Taking conditional expectation $\mathbb{E}_t[\cdot]$ and using the decomposition
	$g_{\Phi_t}=\nabla_\Phi\mathcal{L}_{\text{meta}}(\Phi_t)+\psi^{(t)}+m_t$
	with $\mathbb{E}_t[\psi^{(t)}]=0$, we obtain
	\begin{equation}
		\begin{aligned}
			\mathbb{E}_t\!\left[\mathcal{L}_{\text{meta}}(\Phi_{t+1})-\mathcal{L}_{\text{meta}}(\Phi_t)\right]
			\le\;&
			-\gamma_t \Big\|\nabla_\Phi\mathcal{L}_{\text{meta}}(\Phi_t)\Big\|^2
			-\gamma_t\Big\langle \nabla_\Phi\mathcal{L}_{\text{meta}}(\Phi_t), m_t\Big\rangle \\
			&\quad
			+\frac{L_\Phi}{2}\gamma_t^2\,\mathbb{E}_t\!\left[\|g_{\Phi_t}\|^2\right],
		\end{aligned}
		\label{eq:phi_after_expand_inner}
	\end{equation}
	where the term involving $\psi^{(t)}$ in the inner product vanishes since
	$\mathbb{E}_t\langle \nabla_\Phi\mathcal{L}_{\text{meta}}(\Phi_t),\psi^{(t)}\rangle
	=
	\langle \nabla_\Phi\mathcal{L}_{\text{meta}}(\Phi_t),\mathbb{E}_t[\psi^{(t)}]\rangle
	=0$.
	
	Next, we bound the cross term by Young's inequality:
	\[
	-\gamma_t\langle \nabla_\Phi\mathcal{L}_{\text{meta}}(\Phi_t), m_t\rangle
	\le
	\frac{\gamma_t}{2}\Big\|\nabla_\Phi\mathcal{L}_{\text{meta}}(\Phi_t)\Big\|^2
	+\frac{\gamma_t}{2}\|m_t\|^2 .
	\]
	It remains to bound $\mathbb{E}_t[\|g_{\Phi_t}\|^2]$. Using
	$\|a+b+c\|^2\le 3\|a\|^2+3\|b\|^2+3\|c\|^2$ with
	$a=\nabla_\Phi\mathcal{L}_{\text{meta}}(\Phi_t)$, $b=\psi^{(t)}$, $c=m_t$, we have
	\begin{equation}
		\mathbb{E}_t\!\left[\|g_{\Phi_t}\|^2\right]
		\le
		3\Big\|\nabla_\Phi\mathcal{L}_{\text{meta}}(\Phi_t)\Big\|^2
		+3\mathbb{E}_t\!\left[\|\psi^{(t)}\|^2\right]
		+3\|m_t\|^2 .
		\label{eq:phi_bound_g2}
	\end{equation}
	Plugging the above bounds into \eqref{eq:phi_after_expand_inner} yields
	\begin{equation}
		\begin{aligned}
			\mathbb{E}_t\!\left[\mathcal{L}_{\text{meta}}(\Phi_{t+1})-\mathcal{L}_{\text{meta}}(\Phi_t)\right]
			\le\;&
			-\gamma_t\|\nabla_\Phi\mathcal{L}_{\text{meta}}(\Phi_t)\|^2
			+\frac{\gamma_t}{2}\|\nabla_\Phi\mathcal{L}_{\text{meta}}(\Phi_t)\|^2
			+\frac{\gamma_t}{2}\|m_t\|^2 \\
			&\quad
			+\frac{L_\Phi}{2}\gamma_t^2\Big(
			3\|\nabla_\Phi\mathcal{L}_{\text{meta}}(\Phi_t)\|^2
			+3\mathbb{E}_t[\|\psi^{(t)}\|^2]
			+3\|m_t\|^2\Big) \\
			=\;&
			\Big(-\frac{\gamma_t}{2}+\frac{3L_\Phi}{2}\gamma_t^2\Big)
			\|\nabla_\Phi\mathcal{L}_{\text{meta}}(\Phi_t)\|^2
			+\frac{3L_\Phi}{2}\gamma_t^2\,\mathbb{E}_t[\|\psi^{(t)}\|^2] \\
			&\quad+
			\Big(\frac{\gamma_t}{2}+\frac{3L_\Phi}{2}\gamma_t^2\Big)\|m_t\|^2 .
		\end{aligned}
		\label{eq:phi_bracket_detail}
	\end{equation}
	Finally, using the variance bound $\mathbb{E}_t[\|\psi^{(t)}\|^2]\le \sigma_\Phi^2$
	gives 
	\begin{equation} \begin{aligned} \mathbb{E}_t\!\left[\mathcal{L}_{\text{meta}}(\Phi_{t+1})-\mathcal{L}_{\text{meta}}(\Phi_t)\right] \le\;& \Big(-\frac{\gamma_t}{2}+\frac{3L_\Phi}{2}\gamma_t^2\Big) \big\|\nabla_\Phi\mathcal{L}_{\text{meta}}(\Phi_t)\big\|^2 +\frac{3L_\Phi}{2}\gamma_t^2\sigma_\Phi^2 \\ &+ \Big(\frac{\gamma_t}{2}+\frac{3L_\Phi}{2}\gamma_t^2\Big)\|m_t\|^2 . \end{aligned} \label{eq:phi_bracket_exact_coeff} \end{equation}

	We now invoke the step-size condition at this point.
	Since $\gamma_t \le \frac{1}{6L_\Phi}$, it holds that
	$\frac{3L_\Phi}{2}\gamma_t^2 \le \frac{\gamma_t}{4}$.
	Substituting this bound into \eqref{eq:phi_bracket_exact_coeff} yields
	\begin{equation}
		\mathbb{E}_t\!\left[\mathcal{L}_{\text{meta}}(\Phi_{t+1})-\mathcal{L}_{\text{meta}}(\Phi_t)\right]
		\le
		-\frac{\gamma_t}{4}\big\|\nabla_\Phi\mathcal{L}_{\text{meta}}(\Phi_t)\big\|^2
		+\frac{3L_\Phi}{2}\gamma_t^2\sigma_\Phi^2
		+\Big(\frac{\gamma_t}{2}+\frac{3L_\Phi}{2}\gamma_t^2\Big)\|m_t\|^2 .
		\label{eq:phi_bracket_final}
	\end{equation}

	Now we turn to the second term.
	The inner update is $\theta_{t+1}=\theta_t-\eta_t g_{\theta_t}$, thus
	$\theta_{t+1}(\Phi_t)-\theta_t(\Phi_t)=-\eta_t g_{\theta_t}$.
	By Assumption~1, $\mathcal{L}_{\text{meta}}(\theta)$ is $L$-smooth in $\theta$, so
	\begin{equation}
		\mathcal{L}_{\text{meta}}(\theta_{t+1}(\Phi_t))-\mathcal{L}_{\text{meta}}(\theta_t(\Phi_t))
		\le
		-\eta_t\langle \nabla_{\theta}\mathcal{L}_{\text{meta}}(\theta_t(\Phi_t)), g_{\theta_t}\rangle
		+\frac{L}{2}\eta_t^2\|g_{\theta_t}\|^2.
		\label{eq:theta_smooth}
	\end{equation}
	Using Assumption~6,
	$g_{\theta_t}=\nabla_{\theta}\mathcal{L}_{\text{tr}}(\theta_t,\bar{\mathbf w}(\Phi_t))+\xi^{(t)}$ with
	$\mathbb{E}_t[\xi^{(t)}]=0$ and $\mathbb{E}_t\|\xi^{(t)}\|^2\le\sigma_\theta^2$,
	we have
	\begin{equation}
		\mathbb{E}_t\langle \nabla_{\theta}\mathcal{L}_{\text{meta}}(\theta_t(\Phi_t)), g_{\theta_t}\rangle
		=
		\langle \nabla_{\theta}\mathcal{L}_{\text{meta}}(\theta_t(\Phi_t)),
		\nabla_{\theta}\mathcal{L}_{\text{tr}}(\theta_t,\bar{\mathbf w}(\Phi_t))\rangle,
	\end{equation}
	and
	\begin{equation}
		\mathbb{E}_t\|g_{\theta_t}\|^2
		\le
		\|\nabla_{\theta}\mathcal{L}_{\text{tr}}(\theta_t,\bar{\mathbf w}(\Phi_t))\|^2+\sigma_\theta^2.
		\label{eq:gtheta_moment_again}
	\end{equation}
	Taking expectation in \eqref{eq:theta_smooth} and applying Cauchy--Schwarz,
	together with Assumption~3,
	we get:
	\begin{equation}
		\mathbb{E}_t\!\left[\mathcal{L}_{\text{meta}}(\theta_{t+1}(\Phi_t))-\mathcal{L}_{\text{meta}}(\theta_t(\Phi_t))\right]
		\le
		\eta_t\rho\,\mathbb{E}_t\Big\|\nabla_{\theta}\mathcal{L}_{\text{tr}}(\theta_t,\bar{\mathbf w}(\Phi_t))\Big\|
		+\frac{L}{2}\eta_t^2\Big(\rho^2+\sigma_\theta^2\Big).
		\label{eq:theta_bracket_final}
	\end{equation}
	
	Combining \eqref{eq:decomp_no_I_II}, \eqref{eq:phi_bracket_final} and \eqref{eq:theta_bracket_final} gives
	\begin{equation}
		\begin{aligned}
			\mathbb{E}_t\Big[\mathcal{L}_{\text{meta}}(\theta_{t+1}(\Phi_{t+1}))-\mathcal{L}_{\text{meta}}(\theta_t(\Phi_t))\Big]
			\le&
			-\frac{\gamma_t}{4}\|\nabla_\Phi\mathcal{L}_{\text{meta}}(\Phi_t)\|^2
			+\frac{3L_\Phi}{2}\gamma_t^2\sigma_\Phi^2
			+\Big(\frac{\gamma_t}{2}+\frac{3L_\Phi}{2}\gamma_t^2\Big)\|m_t\|^2\\
			&+
			\eta_t\rho\,\mathbb{E}_t\Big\|\nabla_{\theta}\mathcal{L}_{\text{tr}}(\theta_t,\bar{\mathbf w}(\Phi_t))\Big\|
			+\frac{L}{2}\eta_t^2(\rho^2+\sigma_\theta^2).
		\end{aligned}
		\label{eq:one_step_final}
	\end{equation}
	Summing \eqref{eq:one_step_final} from $t=0$ to $T-1$ and telescoping yields
	\begin{equation}
		\begin{aligned}
			\frac{1}{4}\sum_{t=0}^{T-1}\gamma_t\,\mathbb{E}\|\nabla_\Phi\mathcal{L}_{\text{meta}}(\Phi_t)\|^2
			\le&
			\mathbb{E}\mathcal{L}_{\text{meta}}(\theta_0(\Phi_0))-\mathbb{E}\mathcal{L}_{\text{meta}}(\theta_T(\Phi_T))
			+\frac{3L_\Phi}{2}\sigma_\Phi^2\sum_{t=0}^{T-1}\gamma_t^2 \\
			&+\sum_{t=0}^{T-1}\Big(\frac{\gamma_t}{2}+\frac{3L_\Phi}{2}\gamma_t^2\Big)\mathbb{E}\|m_t\|^2
			+\rho\sum_{t=0}^{T-1}\eta_t\,\mathbb{E}\Big\|\nabla_{\theta}\mathcal{L}_{\text{tr}}(\theta_t,\bar{\mathbf w}(\Phi_t))\Big\|
			+\frac{L}{2}(\rho^2+\sigma_\theta^2)\sum_{t=0}^{T-1}\eta_t^2.
		\end{aligned}
		\label{eq:telescope_final}
	\end{equation}
	Using the lower bound $\mathcal{L}_{\text{meta}}\ge \mathcal{L}_{\inf}$ and
	$\mathbb{E}\|m_t\|^2\le c_b\epsilon^4$ gives
	\begin{equation}
		\begin{aligned}
			\frac{1}{4}\sum_{t=0}^{T-1}\gamma_t\,\mathbb{E}\|\nabla_\Phi\mathcal{L}_{\text{meta}}(\Phi_t)\|^2
			\le&
			\mathbb{E}\mathcal{L}_{\text{meta}}(\theta_0(\Phi_0))-\mathcal{L}_{\inf}
			+\frac{3L_\Phi}{2}\sigma_\Phi^2\sum_{t=0}^{T-1}\gamma_t^2 \\
			&+c_b\epsilon^4\sum_{t=0}^{T-1}\Big(\frac{\gamma_t}{2}+\frac{3L_\Phi}{2}\gamma_t^2\Big)
			+\rho\sum_{t=0}^{T-1}\eta_t\,\mathbb{E}\Big\|\nabla_{\theta}\mathcal{L}_{\text{tr}}(\theta_t,\bar{\mathbf w}(\Phi_t))\Big\|
			+\frac{L}{2}(\rho^2+\sigma_\theta^2)\sum_{t=0}^{T-1}\eta_t^2.
		\end{aligned}
		\label{eq:telescope_with_bias}
	\end{equation}
	Finally, since $\min_{0\le t\le T-1} a_t \le \frac{\sum_{t=0}^{T-1}\gamma_t a_t}{\sum_{t=0}^{T-1}\gamma_t}$ for $a_t\ge 0$,
	dividing \eqref{eq:telescope_with_bias} by $\sum_{t=0}^{T-1}\gamma_t$ yields
	\begin{equation}
		\begin{aligned}
			\min_{0\le t\le T-1}\mathbb{E}\|\nabla_\Phi\mathcal{L}_{\text{meta}}(\Phi_t)\|^2
			\le\;&
			\frac{4\big(\mathbb{E}\mathcal{L}_{\text{meta}}(\theta_0(\Phi_0))-\mathcal{L}_{\inf}\big)}{\sum_{t=0}^{T-1}\gamma_t}
			+\frac{6L_\Phi\sigma_\Phi^2\sum_{t=0}^{T-1}\gamma_t^2}{\sum_{t=0}^{T-1}\gamma_t} \\[0.5ex]
			&+
			\frac{4c_b\epsilon^4\sum_{t=0}^{T-1}\left(\frac{\gamma_t}{2}+\frac{3L_\Phi}{2}\gamma_t^2\right)}{\sum_{t=0}^{T-1}\gamma_t}
			+\frac{4\rho\sum_{t=0}^{T-1}\eta_t\,\mathbb{E}\|\nabla_\theta\mathcal{L}_{\text{tr}}\|}{\sum_{t=0}^{T-1}\gamma_t}
			+\frac{2L(\rho^2+\sigma_\theta^2)\sum_{t=0}^{T-1}\eta_t^2}{\sum_{t=0}^{T-1}\gamma_t}.
		\end{aligned}
		\label{eq:min_bound_general}
	\end{equation}

	From \eqref{eq:min_bound_general}, since $\gamma_t=\gamma$ for all $t$, we have
	\[
	\sum_{t=0}^{T-1}\gamma_t = T\beta.
	\]
	Therefore,
	\begin{equation}
		\begin{aligned}
			\min_{0\le t\le T-1}\mathbb{E}\|\nabla_\Phi\mathcal{L}_{\mathrm{meta}}(\Phi_t)\|^2
			\le\;&
			\frac{1}{T\beta}
			\Bigg[
			4\big(\mathbb{E}\mathcal{L}_{\mathrm{meta}}(\theta_0(\Phi_0))-\mathcal{L}_{\inf}\big)
			+6L_\Phi\sigma_\Phi^2\sum_{t=0}^{T-1}\gamma_t^2
			\\
			&\quad
			+4c_b\epsilon^4\sum_{t=0}^{T-1}\Bigl(\tfrac{\gamma_t}{2}+\tfrac{3L_\Phi}{2}\gamma_t^2\Bigr)
			+4\rho\sum_{t=0}^{T-1}\eta_t\,\mathbb{E}\|\nabla_\theta\mathcal{L}_{\mathrm{tr}}\|
			\\
			&\quad
			+2L(\rho^2+\sigma_\theta^2)\sum_{t=0}^{T-1}\eta_t^2
			\Bigg].
		\end{aligned}
		\label{eq:after_sum_beta}
	\end{equation}
	
	Let
	\[
	\gamma=\min\Bigl\{\frac{1}{6L_\Phi},\frac{c_1}{\sqrt{T}}\Bigr\},
	\]
	which implies
	\[
	\gamma^{-1}=\max\Bigl\{6L_\Phi,\frac{\sqrt{T}}{c_1}\Bigr\}.
	\]
	Substituting this into \eqref{eq:after_sum_beta}, we obtain
	\[
	\begin{aligned}
		\min_{0\le t\le T-1}\mathbb{E}\|\nabla_\Phi\mathcal{L}_{\mathrm{meta}}(\Phi_t)\|^2
		\le\;&
		\frac{1}{T}
		\Bigg[
		4\big(\mathbb{E}\mathcal{L}_{\mathrm{meta}}(\theta_0(\Phi_0))-\mathcal{L}_{\inf}\big)
		+6L_\Phi\sigma_\Phi^2\sum_{t=0}^{T-1}\gamma_t^2
		\\
		&\quad
		+4c_b\epsilon^4\sum_{t=0}^{T-1}\Bigl(\tfrac{\gamma_t}{2}+\tfrac{3L_\Phi}{2}\gamma_t^2\Bigr)
		+4\rho\sum_{t=0}^{T-1}\eta_t\,\mathbb{E}\|\nabla_\theta\mathcal{L}_{\mathrm{tr}}\|
		\\
		&\quad
		+2L(\rho^2+\sigma_\theta^2)\sum_{t=0}^{T-1}\eta_t^2
		\Bigg]
		\max\Bigl\{6L_\Phi,\frac{\sqrt{T}}{c_1}\Bigr\}.
	\end{aligned}
	\]

	Since
	\[
	\lim_{T\to\infty}\sum_{t=0}^{T-1}\gamma_t^2 < \infty,\qquad
	\lim_{T\to\infty}\sum_{t=0}^{T-1}\eta_t^2 < \infty,\qquad
	\lim_{T\to\infty}\sum_{t=0}^{T-1}\eta_t\,\mathbb{E}\|\nabla_\theta\mathcal{L}_{\mathrm{tr}}\| < \infty,
	\]
	all the terms inside the brackets are bounded and can be absorbed into a constant $C>0$, which yields
	\[
	\min_{0\le t\le T-1}\mathbb{E}\|\nabla_\Phi\mathcal{L}_{\mathrm{meta}}(\Phi_t)\|^2
	\le
	\mathcal{O}\!\Bigl(\frac{C}{\sqrt{T}}\Bigr)
	+
	\mathcal{O}(\epsilon^4).
	\]

\end{proof}

\subsection{Convergence Analysis of Training-Objective}

\begin{theorem}
	\label{thm:AR_stationary}
	Suppose Assumptions~1--6 hold.
	Let the learning rates $\{\gamma_t\}_{t=1}^T$ and $\{\eta_t\}_{t=1}^T$ be monotonically
	non-increasing sequences defined as
	\[
	\gamma_t=\min\Big\{\frac{1}{6L_\Phi},\,\frac{c_1}{\sqrt{T}}\Big\},\qquad
	\eta_t=\min\Big\{\frac{1}{L},\,\frac{c_2}{\sqrt{T}}\Big\},
	\]
	for some constants $c_1,c_2>0$ such that
	\[
	\frac{\sqrt{T}}{c_1}\ge L_\Phi,
	\qquad
	\frac{\sqrt{T}}{c_2}\ge L.
	\]
	Then the proposed algorithm guarantees
	\begin{equation}
		\min_{0\le t\le T-1}
		\mathbb{E}\Big[
		\big\|\nabla_\theta \mathcal{L}_{\text{tr}}(\theta_t,\bar{\mathbf w}(\Phi_t))\big\|^2
		\Big]
		\;\le\;
		\mathcal{O}\!\Big(\frac{C}{\sqrt{T}}\Big)
		\;+\;
		\mathcal{O}(\epsilon^2),
	\end{equation}
	where $C$ is a positive constant independent of $T$ and the convergence process.
	Here, $\epsilon$ denotes the finite-difference step size used for gradient
	approximation.
\end{theorem}

\begin{proof}
	From the Preliminaries,
	\[
	\mathcal{L}_{\text{tr}}(\theta_t,\Phi_t):=\mathcal{L}_{\text{tr}}(\theta_t,\bar{\mathbf w}(\Phi_t))
	=\sum_{k=1}^K \bar w_k(\Phi_t)\,\ell_k(\theta_t).
	\]
	We observe:
	\begin{equation}
		\begin{aligned}
			\mathcal{L}_{\text{tr}}(\theta_{t+1},\Phi_{t+1})-\mathcal{L}_{\text{tr}}(\theta_t,\Phi_t)
			=&\ \Big[\mathcal{L}_{\text{tr}}(\theta_{t+1},\Phi_{t+1})-\mathcal{L}_{\text{tr}}(\theta_{t+1},\Phi_t)\Big] \\
			&\ +\Big[\mathcal{L}_{\text{tr}}(\theta_{t+1},\Phi_t)-\mathcal{L}_{\text{tr}}(\theta_t,\Phi_t)\Big].
		\end{aligned}
		\label{eq:AR_decomp_pf}
	\end{equation}
	
	We try to bound the first term:
	\begin{equation}
		\mathcal{L}_{\text{tr}}(\theta_{t+1},\Phi_{t+1})-\mathcal{L}_{\text{tr}}(\theta_{t+1},\Phi_t)
		=\sum_{k=1}^K \ell_k(\theta_{t+1})\Big(\bar w_k(\Phi_{t+1})-\bar w_k(\Phi_t)\Big).
		\label{eq:AR_phi_expand_pf}
	\end{equation}
	Let $\Delta\Phi_t:=\Phi_{t+1}-\Phi_t=-\gamma_t g_{\Phi_t}$,where $g_{\Phi_t}=\nabla_\Phi\mathcal{L}_{\text{meta}}(\Phi_t)+\psi^{(t)}+m_t,\mathbb{E}_t[\psi^{(t)}]=0.$ 
	
	Since $\bar w_k(\Phi)$ is twice differentiable with $B$-bounded Hessian, which yields
	\begin{equation}
		\bar w_k(\Phi_{t+1})-\bar w_k(\Phi_t)
		\le
		\big\langle \nabla_\Phi \bar w_k(\Phi_t),\,\Delta\Phi_t\big\rangle
		+\frac{B}{2}\|\Delta\Phi_t\|^2
		=
		-\gamma_t\big\langle \nabla_\Phi \bar w_k(\Phi_t),\,g_{\Phi_t}\big\rangle
		+\frac{B}{2}\gamma_t^2\|g_{\Phi_t}\|^2 .
		\label{eq:wk_descent_pf}
	\end{equation}
	Substituting \eqref{eq:wk_descent_pf} into \eqref{eq:AR_phi_expand_pf} gives
	\begin{equation}
		\begin{aligned}
			\mathcal{L}_{\text{tr}}(\theta_{t+1},\Phi_{t+1})-\mathcal{L}_{\text{tr}}(\theta_{t+1},\Phi_t)
			\le
			-\gamma_t\sum_{k=1}^K \ell_k(\theta_{t+1})
			\big\langle \nabla_\Phi \bar w_k(\Phi_t),\,g_{\Phi_t}\big\rangle
			+\frac{B}{2}\gamma_t^2\|g_{\Phi_t}\|^2\sum_{k=1}^K \ell_k(\theta_{t+1}).
		\end{aligned}
		\label{eq:AR_phi_before_bdd_pf}
	\end{equation}
	
	Assume bounded component losses $0\le \ell_k(\theta)\le G_\ell$ ,
	so that $\sum_{k=1}^K \ell_k(\theta_{t+1})\le KG_\ell$ according to the assumption 5. Then
	\begin{equation}
		\mathcal{L}_{\text{tr}}(\theta_{t+1},\Phi_{t+1})-\mathcal{L}_{\text{tr}}(\theta_{t+1},\Phi_t)
		\le
		-\gamma_t\sum_{k=1}^K \ell_k(\theta_{t+1})
		\big\langle \nabla_\Phi \bar w_k(\Phi_t),\,g_{\Phi_t}\big\rangle
		+\frac{BKG_\ell}{2}\gamma_t^2\|g_{\Phi_t}\|^2.
		\label{eq:AR_phi_main_pf}
	\end{equation}
	
	Taking expectation of the first term gives
	\begin{equation}
		\begin{aligned}
			-\gamma_t\,\mathbb{E}_t\!\left[\sum_{k=1}^K \ell_k(\theta_{t+1})
			\big\langle \nabla_\Phi \bar w_k(\Phi_t),\,g_{\Phi_t}\big\rangle\right]
			=
			-\gamma_t\sum_{k=1}^K \ell_k(\theta_{t+1})
			\big\langle \nabla_\Phi \bar w_k(\Phi_t),\,\nabla_\Phi\mathcal{L}_{\text{meta}}(\Phi_t)\big\rangle
			-\gamma_t\sum_{k=1}^K \ell_k(\theta_{t+1})
			\big\langle \nabla_\Phi \bar w_k(\Phi_t),\,m_t\big\rangle.
		\end{aligned}
		\label{eq:first_order_split_pf}
	\end{equation}
	Drop the negative sign and apply Cauchy--Schwarz:
	\begin{equation}
		\begin{aligned}
			&-\sum_{k=1}^K \ell_k(\theta_{t+1})
			\big\langle \nabla_\Phi \bar w_k(\Phi_t),\,\nabla_\Phi\mathcal{L}_{\text{meta}}(\Phi_t)\big\rangle
			\le
			\Big\|\sum_{k=1}^K \ell_k(\theta_{t+1})\nabla_\Phi \bar w_k(\Phi_t)\Big\|
			\cdot \big\|\nabla_\Phi\mathcal{L}_{\text{meta}}(\Phi_t)\big\|,\\
			&-\sum_{k=1}^K \ell_k(\theta_{t+1})
			\big\langle \nabla_\Phi \bar w_k(\Phi_t),\,m_t\big\rangle
			\le
			\Big\|\sum_{k=1}^K \ell_k(\theta_{t+1})\nabla_\Phi \bar w_k(\Phi_t)\Big\|
			\cdot \|m_t\|.
		\end{aligned}
		\label{eq:cauchy_pf}
	\end{equation}
	By Assumption~2, $\|\nabla_\Phi \bar{\mathbf w}(\Phi)\|\le \delta$ componentwise, hence
	\begin{equation}
		\Big\|\sum_{k=1}^K \ell_k(\theta_{t+1})\nabla_\Phi \bar w_k(\Phi_t)\Big\|
		\le
		\sum_{k=1}^K \ell_k(\theta_{t+1})\|\nabla_\Phi \bar w_k(\Phi_t)\|
		\le
		KG_\ell\,\delta.
		\label{eq:sum_gradw_pf}
	\end{equation}
	Combining \eqref{eq:first_order_split_pf}--\eqref{eq:sum_gradw_pf} yields
	\begin{equation}
		-\gamma_t\,\mathbb{E}_t\!\left[\sum_{k=1}^K \ell_k(\theta_{t+1})
		\big\langle \nabla_\Phi \bar w_k(\Phi_t),\,g_{\Phi_t}\big\rangle\right]
		\le
		\gamma_t\,KG_\ell\delta\,\big\|\nabla_\Phi\mathcal{L}_{\text{meta}}(\Phi_t)\big\|
		+\gamma_t\,KG_\ell\delta\,\mathbb{E}_t\|m_t\|.
		\label{eq:first_order_final_pf}
	\end{equation}
	
	For the second term in \eqref{eq:AR_phi_main_pf}, we explicitly bound
	the second moment of the meta-gradient estimator.
	By Assumption~3 and Assumption~6, we have
	\[
	\|\nabla_\Phi\mathcal{L}_{\text{meta}}(\Phi_t)\|\le \rho,
	\qquad
	\mathbb{E}_t\|\psi^{(t)}\|^2 \le \sigma_\Phi^2,
	\qquad
	\mathbb{E}_t\|m_t\|^2 \le c_b\epsilon^4.
	\]
	Using the inequality $\|a+b+c\|^2 \le 3\|a\|^2+3\|b\|^2+3\|c\|^2$, it follows that
	\begin{equation}
		\mathbb{E}_t\|g_{\Phi_t}\|^2
		=
		\mathbb{E}_t\|\nabla_\Phi\mathcal{L}_{\text{meta}}(\Phi_t)+\psi^{(t)}+m_t\|^2
		\le
		3\rho^2 + 3\sigma_\Phi^2 + 3c_b\epsilon^4.
		\label{eq:gphi_second_moment_pf}
	\end{equation}
	
	Taking conditional expectation in \eqref{eq:AR_phi_main_pf} and substituting
	\eqref{eq:first_order_final_pf} and \eqref{eq:gphi_second_moment_pf}, we obtain
	\begin{equation}
		\begin{aligned}
			\mathbb{E}_t\!\left[
			\mathcal{L}_{\text{tr}}(\theta_{t+1},\Phi_{t+1})
			-\mathcal{L}_{\text{tr}}(\theta_{t+1},\Phi_t)
			\right]
			\le\;&
			\gamma_t\,KG_\ell\delta\,\|\nabla_\Phi\mathcal{L}_{\text{meta}}(\Phi_t)\|
			+\gamma_t\,KG_\ell\delta\,\mathbb{E}_t\|m_t\| \\
			&+
			\frac{3BKG_\ell}{2}\gamma_t^2
			\big(\rho^2 + \sigma_\Phi^2 + c_b\epsilon^4\big).
		\end{aligned}
		\label{eq:phi_bracket_final_pf}
	\end{equation}

	We now turn to the second term of \eqref{eq:AR_decomp_pf}.
	By Assumption~1, $\mathcal{L}_{\text{tr}}(\cdot,\Phi_t)$ is $L$-smooth in $\theta$, which yields
	\begin{equation}
		\mathcal{L}_{\text{tr}}(\theta_{t+1},\Phi_t)-\mathcal{L}_{\text{tr}}(\theta_t,\Phi_t)
		\le
		-\eta_t\big\langle \nabla_\theta\mathcal{L}_{\text{tr}}(\theta_t,\Phi_t),\,g_{\theta_t}\big\rangle
		+\frac{L}{2}\eta_t^2\|g_{\theta_t}\|^2.
		\label{eq:theta_descent_pf}
	\end{equation}
	By Assumption~6,
	$g_{\theta_t}=\nabla_\theta\mathcal{L}_{\text{tr}}(\theta_t,\Phi_t)+\xi^{(t)}$ with
	$\mathbb{E}_t[\xi^{(t)}]=0$ and $\mathbb{E}_t\|\xi^{(t)}\|^2\le \sigma_\theta^2$, hence
	\begin{equation}
		\mathbb{E}_t\big\langle \nabla_\theta\mathcal{L}_{\text{tr}}(\theta_t,\Phi_t),\,g_{\theta_t}\big\rangle
		=
		\|\nabla_\theta\mathcal{L}_{\text{tr}}(\theta_t,\Phi_t)\|^2,
		\qquad
		\mathbb{E}_t\|g_{\theta_t}\|^2
		\le
		\|\nabla_\theta\mathcal{L}_{\text{tr}}(\theta_t,\Phi_t)\|^2+\sigma_\theta^2.
		\label{eq:gtheta_moments_pf}
	\end{equation}
	Taking conditional expectation in \eqref{eq:theta_descent_pf} and using $\eta_t\le 1/L$ gives
	\begin{equation}
		\mathbb{E}_t\!\left[\mathcal{L}_{\text{tr}}(\theta_{t+1},\Phi_t)-\mathcal{L}_{\text{tr}}(\theta_t,\Phi_t)\right]
		\le
		-\frac{\eta_t}{2}\|\nabla_\theta\mathcal{L}_{\text{tr}}(\theta_t,\Phi_t)\|^2
		+\frac{L}{2}\eta_t^2\sigma_\theta^2.
		\label{eq:theta_bracket_final_pf}
	\end{equation}

	Combine \eqref{eq:AR_decomp_pf}, \eqref{eq:phi_bracket_final_pf}, and
	\eqref{eq:theta_bracket_final_pf}, take full expectation and sum over
	$t=0,\dots,T-1$.
	Using the lower boundedness $\mathcal{L}_{\text{tr}}\ge \mathcal{L}^{\text{tr}}_{\inf}$ yields
	\begin{equation}
		\begin{aligned}
			\frac{1}{2}\sum_{t=0}^{T-1}\eta_t\,
			\mathbb{E}\|\nabla_\theta\mathcal{L}_{\text{tr}}(\theta_t,\Phi_t)\|^2
			\le&
			\ \mathbb{E}\mathcal{L}_{\text{tr}}(\theta_0,\Phi_0)-\mathcal{L}^{\text{tr}}_{\inf}
			+\frac{L}{2}\sigma_\theta^2\sum_{t=0}^{T-1}\eta_t^2 \\
			&+KG_\ell\delta\sum_{t=0}^{T-1}\gamma_t\,
			\mathbb{E}\|\nabla_\Phi\mathcal{L}_{\text{meta}}(\Phi_t)\|
			+KG_\ell\delta\sum_{t=0}^{T-1}\gamma_t\,\mathbb{E}\|m_t\| \\
			&+\frac{3BKG_\ell}{2}\big(\rho^2+\sigma_\Phi^2+c_b\epsilon^4\big)
			\sum_{t=0}^{T-1}\gamma_t^2 .
		\end{aligned}
		\label{eq:telescope_pf}
	\end{equation}
	
	Dividing both sides of \eqref{eq:telescope_pf} by $\sum_{t=0}^{T-1}\eta_t$ and using
	$\min_t a_t \le \frac{\sum_t \eta_t a_t}{\sum_t \eta_t}$ for $a_t\ge 0$, we obtain
	\begin{equation}
		\begin{aligned}
			\min_{0\le t\le T-1}
			\mathbb{E}\|\nabla_\theta\mathcal{L}_{\text{tr}}(\theta_t,\Phi_t)\|^2
			\le\;&
			\frac{2\big(\mathbb{E}\mathcal{L}_{\text{tr}}(\theta_0,\Phi_0)
				-\mathcal{L}^{\text{tr}}_{\inf}\big)}{\sum_{t=0}^{T-1}\eta_t}
			+\frac{L\sigma_\theta^2\sum_{t=0}^{T-1}\eta_t^2}{\sum_{t=0}^{T-1}\eta_t} \\
			&+\frac{2KG_\ell\delta
				\sum_{t=0}^{T-1}\gamma_t\,\mathbb{E}\|\nabla_\Phi\mathcal{L}_{\text{meta}}(\Phi_t)\|}
			{\sum_{t=0}^{T-1}\eta_t}
			+\frac{2KG_\ell\delta\,\sqrt{c_b}\,\epsilon^2\sum_{t=0}^{T-1}\gamma_t}
			{\sum_{t=0}^{T-1}\eta_t} \\
			&+\frac{3BKG_\ell\big(\rho^2+\sigma_\Phi^2+c_b\epsilon^4\big)
				\sum_{t=0}^{T-1}\gamma_t^2}{\sum_{t=0}^{T-1}\eta_t}.
		\end{aligned}
		\label{eq:theta_stationary_pre}
	\end{equation}
	where we used $\mathbb{E}\|m_t\|\le \sqrt{\mathbb{E}\|m_t\|^2}\le \sqrt{c_b}\epsilon^2$.

	From \eqref{eq:theta_stationary_pre}, since $\eta_t=\eta$ for all $t$, we have
	\[
	\sum_{t=0}^{T-1}\eta_t = T\eta.
	\]
	Therefore,
	\begin{equation}
		\begin{aligned}
			\min_{0\le t\le T-1}
			\mathbb{E}\|\nabla_\theta\mathcal{L}_{\mathrm{tr}}(\theta_t,\Phi_t)\|^2
			\le\;&
			\frac{1}{T\eta}
			\Bigg[
			2\big(\mathbb{E}\mathcal{L}_{\mathrm{tr}}(\theta_0,\Phi_0)
			-\mathcal{L}^{\mathrm{tr}}_{\inf}\big)
			+L\sigma_\theta^2\sum_{t=0}^{T-1}\eta_t^2
			\\
			&\quad
			+2KG_\ell\delta
			\sum_{t=0}^{T-1}\gamma_t\,\mathbb{E}\|\nabla_\Phi\mathcal{L}_{\mathrm{meta}}(\Phi_t)\|
			+2KG_\ell\delta\,\sqrt{c_b}\,\epsilon^2\sum_{t=0}^{T-1}\gamma_t
			\\
			&\quad
			+3BKG_\ell\big(\rho^2+\sigma_\Phi^2+c_b\epsilon^4\big)
			\sum_{t=0}^{T-1}\gamma_t^2
			\Bigg].
		\end{aligned}
		\label{eq:theta_after_sum_eta}
	\end{equation}
	
	Let
	\[
	\eta=\min\Bigl\{\frac{1}{L},\frac{c_2}{\sqrt{T}}\Bigr\},
	\]
	which implies
	\[
	\eta^{-1}=\max\Bigl\{L,\frac{\sqrt{T}}{c_2}\Bigr\}.
	\]
	Substituting this into \eqref{eq:theta_after_sum_eta}, we obtain
	\[
	\begin{aligned}
		\min_{0\le t\le T-1}
		\mathbb{E}\|\nabla_\theta\mathcal{L}_{\mathrm{tr}}(\theta_t,\Phi_t)\|^2
		\le\;&
		\frac{1}{T}
		\Bigg[
		2\big(\mathbb{E}\mathcal{L}_{\mathrm{tr}}(\theta_0,\Phi_0)
		-\mathcal{L}^{\mathrm{tr}}_{\inf}\big)
		+L\sigma_\theta^2\sum_{t=0}^{T-1}\eta_t^2
		\\
		&\quad
		+2KG_\ell\delta
		\sum_{t=0}^{T-1}\gamma_t\,\mathbb{E}\|\nabla_\Phi\mathcal{L}_{\mathrm{meta}}(\Phi_t)\|
		+2KG_\ell\delta\,\sqrt{c_b}\,\epsilon^2\sum_{t=0}^{T-1}\gamma_t
		\\
		&\quad
		+3BKG_\ell\big(\rho^2+\sigma_\Phi^2+c_b\epsilon^4\big)
		\sum_{t=0}^{T-1}\gamma_t^2
		\Bigg]
		\max\Bigl\{L,\frac{\sqrt{T}}{c_2}\Bigr\}.
	\end{aligned}
	\]
	
	Since
	\[
	\lim_{T\to\infty}\sum_{t=0}^{T-1}\eta_t^2 < \infty,\qquad
	\lim_{T\to\infty}\sum_{t=0}^{T-1}\gamma_t^2 < \infty,\qquad
	\lim_{T\to\infty}\sum_{t=0}^{T-1}\gamma_t\,
	\mathbb{E}\|\nabla_\Phi\mathcal{L}_{\mathrm{meta}}(\Phi_t)\| < \infty,
	\]
	all the terms inside the brackets are bounded and can be absorbed into a constant
	$C>0$, which yields
	\[
	\min_{0\le t\le T-1}
	\mathbb{E}\|\nabla_\theta\mathcal{L}_{\mathrm{tr}}(\theta_t,\Phi_t)\|^2
	\le
	\mathcal{O}\!\Bigl(\frac{C}{\sqrt{T}}\Bigr)
	+
	\mathcal{O}(\epsilon^2).
	\]
	
\end{proof}

\section{Related Work}
\label{sec:related_work}

\subsection{Domain Weighting in Machine Learning}

The optimization of data composition through domain weighting has evolved from sample re-weighting in traditional supervised learning to the complex scheduling challenges of the Large Language Model (LLM) era.
Weighting training samples is a long-standing problem in machine learning, primarily addressed through meta-learning and bi-level optimization (BLO) \cite{ren2018learning, shu2019metaweightnet}. These frameworks establish a rigorous data-driven foundation by learning an importance-weighting policy that minimizes validation loss to handle issues like label noise or class imbalance \cite{wang2020tanda, koh2017understanding}. While theoretically robust in manageable data regimes, the immense scale of LLM pre-training datasets rendered these direct bi-level applications computationally prohibitive. Consequently, early LLMs such as GPT-3 \cite{brown2020language} and Llama 2 \cite{touvron2023llama2} initially reverted to manual heuristics, where domain sampling proportions were determined by human intuition based on perceived dataset quality.

To move beyond manual tuning, recent research has focused on exploring automated weighting strategies on smaller-scale proxy models through two primary paths~\cite{albalak2024survey, liu2025rethinking, chen2024datajuicer, xie2023doremi, liu2024regmix, kang2024autoscale, ye2024datamixing, ge2024bimix, belenki2025mde, hoffmann2022chinchilla, gu2025data, shukor2025scaling, wettig2025organize, xie2025chameleon, thudi2025mixmin, corrado2025automixalign, li2025pike, bansal2025honeybee}. The first path is direct optimization, exemplified by DoReMi \cite{xie2023doremi} and DoGE \cite{fan2024doge}, which solve for domain weights directly on the target scale. However, these methods are notoriously unstable due to extreme batch-level variance and incur massive computational overhead, often requiring exhaustive data traversal to achieve convergence. The second path is assumption-based optimization (or function-fitting), which introduces structural constraints such as rank invariance \cite{liu2024regmix, belenki2025mde} or scaling laws \cite{kang2024autoscale} to map the search onto small-data regimes. Although efficient, these structural assumptions are frequently violated in practice, leading to significant estimation bias and suboptimal mixtures. ByDoRe bridges this gap by shifting the optimization paradigm from deterministic fixed modeling to a probabilistic framework, achieving superior stability and lower computational cost by effectively suppressing variance-induced jitters.

\subsection{Probabilistic Modeling and Prior Learning}
\label{sec:related_probabilistic}

The transition from deterministic point estimation to probabilistic frameworks represents a fundamental evolution in addressing optimization instability within complex systems. Traditional point-wise methods, centered on Maximum Likelihood Estimation (MLE), seek to identify singular numerical values for parameters but remain highly vulnerable to extreme batch-level variance, which often leads to violent optimization jitters. To suppress these instabilities, Bayesian frameworks treat parameters as latent stochastic variables governed by prior distributions, providing a principled mechanism to internalize structural regularities and account for uncertainty \cite{gelman2013bayesian, efron2012large, kendall2018multi}.

Hierarchical Bayesian models, such as the Dirichlet process \cite{ferguson1973bayesian} and Latent Dirichlet Allocation (LDA) \cite{blei2003latent}, have been extensively utilized to model multi-level dependencies and latent structures. By introducing hierarchical priors like the Gamma-Dirichlet structure, these models provide structural regularization to capture long-term data patterns rather than reacting to transient variance. Despite their theoretical robustness, Bayesian methods are conventionally perceived as more computationally expensive than deterministic optimization due to the overhead of posterior inference or iterative hyperparameter search \cite{snoek2012practical}. ByDoRe addresses these limitations by shifting the domain weighting task from deterministic modeling to a probabilistic framework. By intervening in the data sampling process to learn prior-assignment laws, it achieves stable weight identification with significantly lower computational overhead than traditional paradigms.

\section{Dataset Information}
\label{app:dataset_details}

In our experiments, we utilize the 17 publicly available domains from \textit{The Pile} dataset \cite{gao2020pile}. These domains represent a wide variety of data distributions, ranging from technical research papers and source code to conversational logs and legal documents. As shown in Table~\ref{tab:pile_stats_full}, we exclude components such as Books3 and OpenWebText2 due to documented copyright concerns, focusing on the remaining subsets to ensure the reproducibility of our prior-setting laws.

\begin{table}[ht]
	\centering
	\caption{Detailed overview of the 17 available domains in The Pile used in our experiments.}
	\label{tab:pile_stats_full}
	\small
	\begin{tabular}{lrl}
		\toprule
		\textbf{Component} & \textbf{Effective Size} & \textbf{Brief Description} \\
		\midrule
		Pile-CC            & 227.12 GiB & Web crawl data processed via the Pile's extraction pipeline. \\
		PubMed Central     & 180.55 GiB & Full-text biomedical research articles from the NIH. \\
		ArXiv              & 112.42 GiB & Scientific preprints in Physics, Mathematics, and Computer Science. \\
		GitHub             & 95.16 GiB  & Public source code repositories from various programming languages. \\
		FreeLaw            & 76.73 GiB  & Legal court opinions from US federal and state courts. \\
		Stack Exchange     & 64.39 GiB  & Q\&A data from the diverse Stack Exchange network. \\
		USPTO Backgrounds  & 45.81 GiB  & Technical background sections of US patents. \\
		PubMed Abstracts   & 38.53 GiB  & Summaries of biomedical research papers. \\
		Gutenberg (PG-19)  & 27.19 GiB  & Long-form literary works from the Project Gutenberg library. \\
		Wikipedia (en)     & 19.13 GiB  & High-quality encyclopedic content from English Wikipedia. \\
		DM Mathematics     & 15.49 GiB  & Generated mathematics problems covering various topics. \\
		Ubuntu IRC         & 11.03 GiB  & Chat logs from the Ubuntu technical support channels. \\
		EuroParl           & 9.17 GiB   & Multilingual proceedings of the European Parliament. \\
		HackerNews         & 7.80 GiB   & Discussion threads and comments from the technology forum. \\
		PhilPapers         & 4.76 GiB   & Academic research papers and books in Philosophy. \\
		NIH ExPorter       & 3.79 GiB   & Descriptions of NIH-funded research projects and grants. \\
		Enron Emails       & 1.76 GiB   & Historical email communication corpus from Enron Corporation. \\
		\bottomrule
	\end{tabular}
\end{table}

\section{Detailed Evaluation Benchmark Information}
\label{app:validation_datasets}

In strict adherence to the evaluation protocol established in \textit{RegMix} \cite{liu2024regmix}, we utilize 13 diverse downstream benchmarks to serve as the validation signals for domain weight identification. These benchmarks span common sense reasoning, scientific knowledge, and logical inference.

\paragraph{Social IQA \cite{sap2019social}} This benchmark focuses on social commonsense, requiring the model to reason about social interactions and the emotional states of individuals. It evaluates the "theory of mind" and pragmatic reasoning capabilities of pre-trained models.

\paragraph{HellaSwag \cite{zellers2019hellaswag}} A challenge dataset for common sense natural language inference, requiring models to predict the most plausible continuation of a scenario. It is a critical metric for evaluating a model's understanding of physical intuition and daily events.

\paragraph{PIQA \cite{bisk2020piqa}} The Physical Interaction QA dataset assesses common sense reasoning regarding physical objects and their interactions. It tests the model's ability to generalize implicit physical knowledge that is rarely stated explicitly in text.

\paragraph{OpenBookQA \cite{mihaylov2018can}} Modeled after open-book exams, this dataset requires the model to combine provided facts with external common sense knowledge. It evaluates the model's ability to perform multi-hop reasoning and knowledge integration.

\paragraph{Lambada \cite{paperno2016lambada}} This task evaluates the model's ability to predict the last word of a sentence based on a broad context. It is a rigorous test of long-range dependency modeling and discourse-level understanding.

\paragraph{SciQ \cite{welbl2017crowdsourcing}} SciQ contains crowdsourced science exam questions covering physics, chemistry, and biology. It assesses the model's capacity to retrieve and reason over specialized scientific domain knowledge.

\paragraph{ARC Easy \cite{clark2018think}} Part of the AI2 Reasoning Challenge, the "Easy" set consists of multiple-choice science questions from elementary exams. It serves as a baseline for factual retrieval and basic reasoning in the scientific domain.

\paragraph{COPA \cite{sarlin2020superglue}} The Choice Of Plausible Alternatives evaluates the model's ability to identify causal relationships between events. It challenges the model to perform high-level semantic reasoning regarding cause and effect.

\paragraph{RACE \cite{lai2017race}} Sourced from English exams for middle and high school students in China, RACE requires comprehensive reading comprehension. It tests the model's ability to extract information and reason over long, complex passages.

\paragraph{LogiQA \cite{liu2020logiqa}} Derived from the National Civil Servants Examination, this dataset focuses on logical reasoning. it requires the model to perform deductive and inductive reasoning to solve complex logical puzzles.

\paragraph{QQP \cite{wang2018glue}} The Quora Question Pairs task requires the model to determine whether two questions are semantically equivalent. It is a standard measure for evaluating paraphrase detection and semantic similarity.

\paragraph{WinoGrande \cite{sakaguchi2021winogrande}} A large-scale dataset for commonsense reasoning designed to be robust against statistical biases. It tests coreference resolution capabilities based on deep semantic and contextual understanding.

\paragraph{MultiRC \cite{khashabi2018multirc}} A multi-sentence reading comprehension dataset where each question may have multiple correct answers. It evaluates the model’s ability to synthesize information across multiple sentences to perform complex reasoning.

\section{Detailed Settings of Data Scheduling Strategies}
\label{app:strategies_setting}

To investigate the impact of how domain weights $\mathbf{w}$ are physically realized during the optimization process, we define and evaluate four distinct scheduling strategies. Let $B$ denote the total batch size (in sequences or tokens), $K$ the number of domains, and $x_{i,k}$ the $i$-th sequence sampled from domain $D_k$. The implementation details for each strategy are as follows:

\paragraph{(i) Intra-batch weighted sampling (ByDoRe Standard)} 
In this configuration, the domain weights directly determine the composition of every training batch. For each update step, the batch $\mathcal{B}$ is constructed by drawing $n_k = \lfloor w_k \cdot B \rfloor$ sequences from each domain $D_k$. The model is updated using the standard auto-regressive loss:
\begin{equation}
	\mathcal{L}_{\text{intra-samp}} = \frac{1}{B} \sum_{k=1}^K \sum_{i=1}^{n_k} \ell_\theta(x_{i,k}).
\end{equation}
This strategy ensures that the gradient at every step is a physical realization of the mixture $p_{\mathbf{w}}$, providing the most direct feedback to the prior network.

\paragraph{(ii) Inter-batch weighted sampling} 
Here, each individual batch is homogeneous, containing sequences from only one domain. The weights $\mathbf{w}$ define a categorical distribution $\text{Cat}(\mathbf{w})$ from which a domain $k^*$ is sampled for the entire batch. The update is:
\begin{equation}
	\mathcal{B} \sim D_{k^*}, \text{ where } k^* \sim \text{Cat}(\mathbf{w}); \quad \mathcal{L}_{\text{inter-samp}} = \frac{1}{B} \sum_{i=1}^{B} \ell_\theta(x_{i,k^*}).
\end{equation}
While this maintains the correct expectation over many steps, the high variance between consecutive batches can destabilize the meta-gradient estimation.

\paragraph{(iii) Intra-batch loss re-weighting} 
This strategy utilizes uniform sampling to construct the batch, meaning each domain is represented by $n = B/K$ sequences regardless of $\mathbf{w}$. The domain importance is instead applied to the loss objective:
\begin{equation}
	\mathcal{L}_{\text{intra-weight}} = \sum_{k=1}^K w_k \cdot \left( \frac{1}{n} \sum_{i=1}^{n} \ell_\theta(x_{i,k}) \right).
\end{equation}
Although mathematically similar in expectation to strategy (i), it fails to prioritize high-quality domains in the physical data reading pipeline, leading to suboptimal gradient signals.

\paragraph{(iv) Inter-batch loss re-weighting} 
Similar to strategy (ii), batches are sampled from domains uniformly ($k^* \sim \text{unif}\{1, \dots, K\}$), but the resulting gradient update is scaled by the corresponding weight $w_{k^*}$. The effective loss is:
\begin{equation}
	\mathcal{B} \sim D_{k^*}; \quad \mathcal{L}_{\text{inter-weight}} = w_{k^*} \cdot K \cdot \left( \frac{1}{B} \sum_{i=1}^{B} \ell_\theta(x_{i,k^*}) \right).
\end{equation}
The factor $K$ is used to normalize the expected gradient scale. This approach typically suffers from the most severe optimization noise as it combines inter-batch variance with the volatility of the weight updates.

\section{Analysis of Scaling Law and Rank Invariance Violations}
\label{app:violation}

\begin{figure}[htbp]
	\centering
	\includegraphics[width=0.5\textwidth]{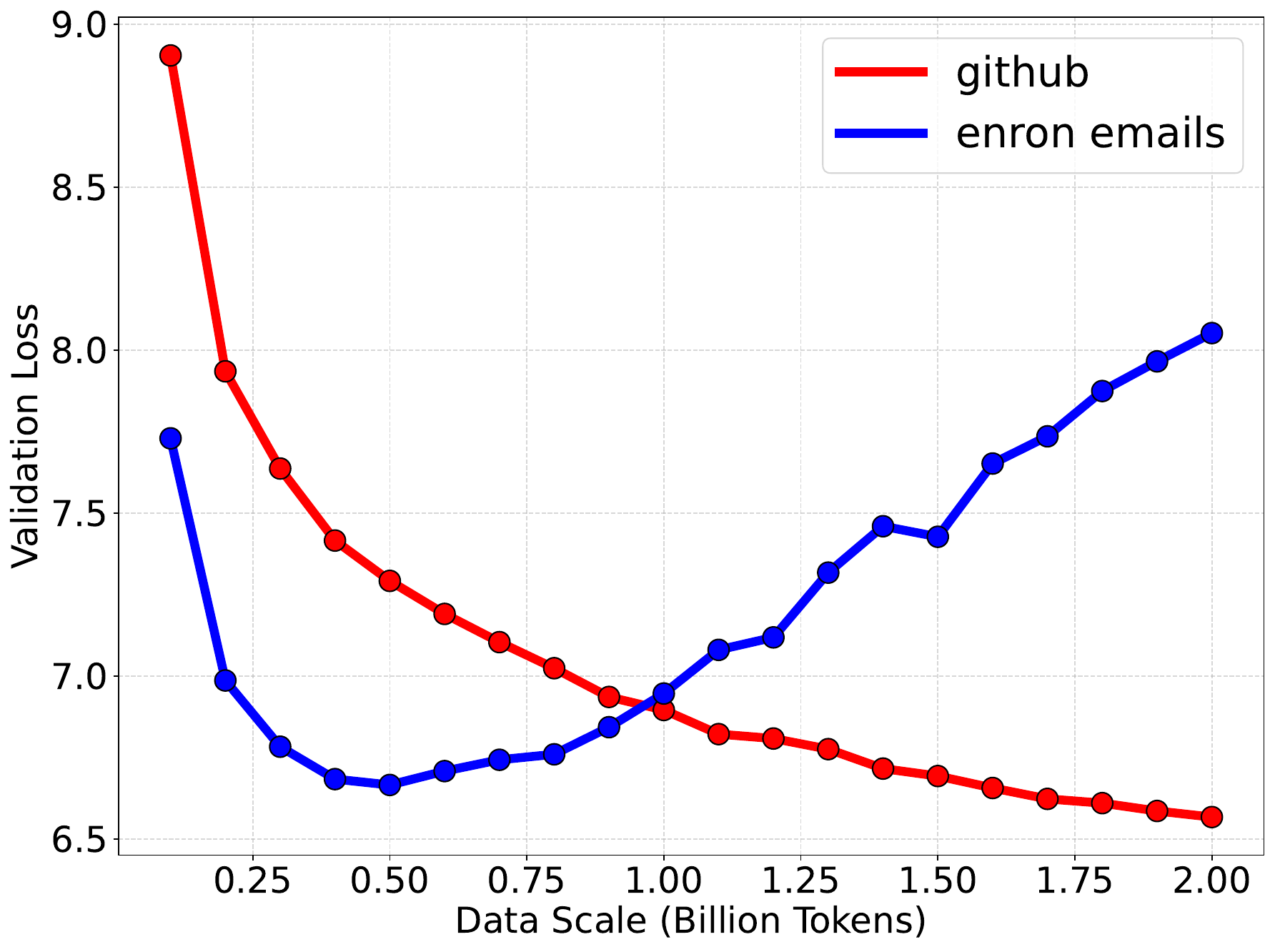}
	\caption{{Validation loss trajectories for Case 2.} The Enron Emails domain exhibits non-monotonic behavior and a rank reversal relative to GitHub at $\sim$0.95B tokens, illustrating the breakdown of standard scaling law assumptions.}
	\label{fig:appendix_loss}
\end{figure}

In this section, we provide a detailed examination of the validation loss trajectories in Case 2 (Figure~\ref{fig:appendix_loss}), which serves as the empirical basis for the performance collapse of existing function-fitting baselines.

Standard domain weighting methods, such as RegMix and AutoScale, predominantly rely on two fundamental assumptions: (1) \textit{Monotonicity}, where loss follows a predictable power-law decay (Scaling Law); and (2) \textit{Rank Invariance}, where the relative difficulty of domains remains consistent across different data scales. Our empirical results in Figure~\ref{fig:appendix_loss} demonstrate a clear violation of both:
\begin{itemize}
	\item \textbf{Violation of Scaling Laws:} While the \textbf{GitHub} domain (red line) exhibits a standard monotonic decrease, the Enron Emails domain (blue line) shows a non-monotonic U-shaped trajectory. The validation loss decreases until approximately 0.5B tokens but subsequently rises as more tokens are consumed. This trend contradicts the basic power-law assumption $\ell(n) \propto n^{-\alpha}$, rendering rigid fitting models incapable of accurate weight estimation.
	\item \textbf{Breakdown of Rank Invariance:} A significant rank reversal occurs at approximately 0.95B tokens. Prior to this crossover, Enron Emails is the lower-loss (easier) domain; however, as training progresses, GitHub becomes the easier domain. Methods that freeze domain weights based on early-stage fitting (like RegMix) fail to adapt to this dynamic shift in domain importance.
\end{itemize}
The failure of RegMix, AutoScale, and MDE in Case 2 (Table~\ref{tab:case2_results}) stems from their reliance on the aforementioned rigid assumptions. When the Enron Emails domain violates the scaling law stability, these methods assign sub-optimal weights based on an incorrect extrapolation of the loss curve.

\end{document}